\documentclass[11pt]{article}

\usepackage[margin=1in]{geometry}

\usepackage{amsmath,amsfonts,bm}

\def\eqref#1{equation~\ref{#1}}

\def\plaineqref#1{\ref{#1}}

\def\1{\bm{1}}

\DeclareMathAlphabet{\mathsfit}{\encodingdefault}{\sfdefault}{m}{sl}
\SetMathAlphabet{\mathsfit}{bold}{\encodingdefault}{\sfdefault}{bx}{n}

\usepackage{amsmath,amssymb,amsthm,mathtools}
\usepackage{bm}
\usepackage{natbib}
\usepackage{booktabs}
\usepackage{array,tabularx}
\usepackage{placeins}
\usepackage{float}
\usepackage{algorithm}
\usepackage{algpseudocode}
\usepackage{graphicx}
\usepackage{microtype}
\usepackage{xcolor}
\usepackage{hyperref}
\hypersetup{
  hidelinks,
  pdftitle={Iterative Exact Discrete Guidance for Energy-Based Sampling},
  pdfauthor={Yuwen Qian, Yidong Ouyang, Zhengyan Wan, Hongyuan Zha},
  pdfsubject={Machine Learning},
  pdfkeywords={discrete sampling, energy-based models, discrete diffusion}
}
\usepackage{url}

\newcommand{\state}{\mathcal{X}}
\newcommand{\alphabet}{\mathcal{S}}
\newcommand{\refdist}{\pi_{\mathrm{ref}}}
\newcommand{\targetdist}{\pi}
\newcommand{\tiltop}{\mathcal{T}}
\newcommand{\tv}{\operatorname{TV}}
\newcommand{\kld}{\operatorname{KL}}
\newcommand{\osc}{\operatorname{osc}}
\newcommand{\ress}{\mathrm{rESS}}
\newcommand{\supp}{\operatorname{supp}}
\newcommand{\indc}[1]{\mathbf{1}\!\left\{#1\right\}}
\newcommand{\dd}{\mathop{}\!\mathrm{d}}

\newtheorem{theorem}{Theorem}
\newtheorem{proposition}{Proposition}
\newtheorem{lemma}{Lemma}
\newtheorem{corollary}{Corollary}
\theoremstyle{definition}

\theoremstyle{remark}

\numberwithin{equation}{section}

\newcommand{\authmark}[1]{\raisebox{0.65ex}{\scriptsize #1}}

\title{Iterative Exact Discrete Guidance for Energy-Based Sampling}

\author{%
  Yuwen Qian\authmark{1}, Yidong Ouyang\authmark{2},
  Zhengyan Wan\authmark{2}, and Hongyuan Zha\authmark{3,*}\\[0.5em]
  {\small \authmark{1}School of Science and Engineering,
    The Chinese University of Hong Kong, Shenzhen}\\
  {\small \authmark{2}Department of Statistics,
    University of California, Los Angeles}\\
  {\small \authmark{3}School of Data Science,
    The Chinese University of Hong Kong, Shenzhen}\\
  {\small \authmark{*}Corresponding author}%
}
\date{}

\begin{document}

\maketitle

\setlength{\parindent}{0pt}
\setlength{\parskip}{0.5pc}

\begin{abstract}
Sampling from unnormalized distributions over large discrete state spaces
becomes difficult when a multimodal target is far from a tractable reference.
We introduce \emph{Iterative Exact Discrete Guidance} (IEDG), a population-exact,
trajectory-wise guidance framework for unnormalized discrete targets. Rather than learn the full reference-to-target correction in one step, IEDG introduces a global Boltzmann tilt along an annealing trajectory. Each stage learns a stage-local posterior correction for an incremental Boltzmann tilt of the current source, while the resulting corrections are accumulated relative to a fixed analytic posterior. At the population optimum, exact stage posteriors recover the correct reverse dynamics, whose exact simulation reproduces the target distribution.

IEDG chooses stage increments by relative effective sample size (rESS), which
controls R\'enyi-$2$ displacement and
locally adapts the step size to the thermodynamic geometry of the annealing path. Our stagewise total-variation analysis shows that limited overlap amplifies Bregman fitting
error by $1/\sqrt{\mathrm{rESS}}$, while posterior, simulation, and truncation
errors enter additively. IEDG improves all distribution-level errors over the neural baselines on ordered, exactly enumerated Ising $4\times4$, while substantially
reducing one-shot errors on Ising/Potts $16\times16$ across thermodynamic
regimes and attaining the best neural-sampler result on several reported
local-statistic and phase-coverage metrics. On
Max-Cut, its best-of-$512$ and average-sample ratios exceed
all the baselines. Code and artifacts are available at
\url{https://github.com/StillFantasy123/iterative-exact-discrete-guidance}.
\end{abstract}

\section{Introduction}
\label{sec:introduction}

Many problems in statistical physics, probabilistic inference, and
combinatorial optimization specify a target only through an unnormalized
energy, $\targetdist(x)\propto\refdist(x)e^{-E(x)}$.
Unlike energy minimization, sampling must reproduce how probability mass is
distributed across configurations, including uncertainty, coexisting phases,
and diverse near-optimal solutions.

Discrete diffusion and flow models amortize transport from a simple reference
by learning reverse rates or endpoint posteriors
\citep{austin2021structured,sun2023score,campbell2024generative,lou2024discrete}.
Recent neural samplers train them directly from unnormalized energies
\citep{ou2025dnfs,zhu2025mdns}. A recurring obstacle is the
global source-to-target tilt: poor overlap concentrates its density ratio on rare
source states, destabilizing learning and mode coverage. Progressive methods
replace this global transport with a sequence of easier problems
\citep{guo2026pdns}.

Discrete Guidance Matching (DGM) shows that an endpoint reweighting can be
converted into an exact correction of a discrete flow's denoising posterior
\citep{wan2026discrete}. Specifically, the target posterior is obtained by
reweighting with a conditional expectation of the
terminal density ratio under the source process. However, when applied directly to energy-based sampling, it must learn the entire reference-to-target correction and becomes difficult under poor overlap.

We introduce \emph{Iterative Exact Discrete Guidance} (IEDG), a
trajectory-wise extension of guidance matching. We first build a global Boltzmann
tilt along an annealing trajectory. At each stage, IEDG learns a mild energy correction from the current sampler and rewrites the updated posterior relative to the original analytic reference. This construction folds all previous corrections into one guide, so terminal inference does not replay the annealing chain. Ideally, each stage uses its source's exact posterior, so the stagewise Bayes updates compose exactly; Section~\ref{sec:theory} bounds the error from instead using the previous learned posterior, together with fitting, simulation, and truncation errors.

IEDG chooses each annealing increment by relative effective sample size
(rESS), taking shorter steps when the current source overlaps poorly with the
next target. The same overlap also controls stability: for the same error in learning a stagewise guidance correction, the resulting sampling error scales as $1/\sqrt{\mathrm{rESS}}$. Thus rESS links path design directly to the propagation of learning error, and our stagewise stability recursion carries these local errors toward the terminal distribution.

\paragraph{Contributions.}
Our main contributions are:
\begin{itemize}
    \item We introduce IEDG, a trajectory-wise guidance framework that replaces a difficult one-shot DGM correction with a sequence of variance-controlled local tilts, each applied as a posterior correction relative to the current source. Its ideal population updates compose exactly, while terminal inference retains only one accumulated guide.

    \item We develop an ESS-controlled annealing path. Population rESS defines a
    finite-step R\'enyi-$2$ trust region, locally recovers equal
    thermodynamic-length increments and quantifies how limited overlap amplifies stagewise Bregman fitting error in multistage total-variation.

    \item Across exact-distribution recovery, lattice sampling, and graph optimization,
    IEDG substantially improves one-shot guidance on enumerated Ising $4\times4$ and
    Max-Cut, while remaining competitive across magnetization, correlation,
    and phase-coverage metrics on Ising and Potts $16\times16$, with improvements on several local-statistic and phase-coverage metrics.
\end{itemize}

\section{Related Work}
\label{sec:related-work}

\paragraph{Discrete diffusion, flow, and guidance.}
Discrete diffusion learns reverse kernels, rates, or density ratios, while
Discrete Flow Matching marginalizes endpoint-conditional rates using a learned
posterior
\citep{austin2021structured,sun2023score,lou2024discrete,campbell2024generative,gat2024discrete}.
Guidance can instead reweight discrete reverse dynamics through classifiers or
CTMC rate corrections \citep{schiff2025simple,nisonoff2025unlocking}. Our
closest foundation is DGM, whose Bayes identity gives an exact posterior
correction from a terminal density ratio \citep{wan2026discrete}.

\paragraph{Neural samplers for unnormalized discrete targets.}
Existing methods use variational objectives, path-space importance weights,
Kolmogorov residuals, or stochastic optimal control
\citep{wu2019solving,sanokowski2025sdds,holderrieth2025leaps,ou2025dnfs,zhu2025mdns,du2026metadns,guo2026dasbs}.
IEDG instead learns conditional incremental tilts and accumulates them relative
to a fixed analytic posterior base.

\paragraph{Progressive transport.}
AIS and SMC bridge targets by reweighting and Markov transitions, with ESS-based
temperature adaptation being well established
\citep{neal2001annealed,delmoral2006sequential,zhou2016automatic}.
Learned variants fit annealed maps or proximal path-space updates
\citep{arbel2021annealed,matthews2022continual,guo2026pdns}. IEDG uses stages to
construct posterior corrections rather than a composition of transport maps;
terminal sampling evaluates one accumulated guide. Appendix~\ref{app:related-work}
gives detailed comparisons.

\section{Preliminaries}
\label{sec:preliminaries}

\subsection{Discrete energy-based sampling}

Let $\alphabet$ be a finite alphabet and $\state=\alphabet^D$. For any
distribution $P$ on $\state$ and positive weight function
$w:\state\to(0,\infty)$, define the normalized tilt of $P$ by
\begin{equation*}
    \tiltop_w(P)(x)
    :=
    \frac{P(x)w(x)}
    {\sum_{x'\in\state}P(x')w(x')}.
\end{equation*}
Given a tractable reference distribution $\refdist$ and an energy function
$E:\state\to\mathbb R$, our goal is to sample from the energy-tilted
distribution
\begin{equation*}
    \targetdist(x)
    =
    \frac{\refdist(x)\exp\{-E(x)\}}
    {\sum_{x'\in\state}\refdist(x')\exp\{-E(x')\}}.
\end{equation*}
Thus $\targetdist=\tiltop_{\exp\{-E\}}(\refdist)$. We assume access to samples from $\refdist$ and pointwise evaluations of $E$,
but neither samples from $\targetdist$ nor its partition function. We write
$x^d$ for coordinate $d$ of $x$. In our experiments,
$\nu=\operatorname{Unif}(\alphabet)$ and
$\refdist=\nu^{\otimes D}$, where $\nu^{\otimes D}$ denotes the $D$-fold
product distribution.

\subsection{Discrete flow models}

Following discrete flow matching
\citep{campbell2024generative,wan2026discrete}, let $q_1$ be a terminal
distribution and let $q_{t\mid 1}(x_t\mid x_1)$ be a conditional probability path
from a noise law $q_0$ at $t=0$ to the endpoint $x_1$ at $t=1$. Its marginal and
posterior are
\begin{align}
    q_t(x_t)
    &=\sum_{x_1\in\state} q_{t\mid 1}(x_t\mid x_1)q_1(x_1),
    \qquad
    q_{1\mid t}(x_1\mid x_t)
    =\frac{q_{t\mid 1}(x_t\mid x_1)q_1(x_1)}{q_t(x_t)}.
    \notag
\end{align}
We write
$q_{1\mid t}^{d}(z\mid x_t):=\mathbb P(X_1^d=z\mid X_t=x_t)$ for the $d$th
coordinate posterior; the full posterior need not factorize. Marginalizing
endpoint-conditional transition rates under $q_{1\mid t}$ yields a CTMC following
$q_t$ and, for coordinate-factorized paths, requires only these coordinate
posteriors \citep{campbell2024generative}. Appendix~\ref{app:background} gives the
complete rate construction.

As in DGM, we use the uniform-replacement path. Let
$\kappa:[0,1]\to[0,1]$ be differentiable and nondecreasing with
$\kappa_0=0$ and $\kappa_1=1$, and let $\nu$ be a full-support categorical
distribution. The conditional path factorizes across coordinates:
\begin{equation*}
    q_{t\mid1}(x_t\mid x_1)
    =\prod_{d=1}^{D}q_{t\mid1}^{d}(x_t^d\mid x_1^d),
    \qquad
    q_{t\mid1}^{d}(x_t^d\mid x_1^d)
    =(1-\kappa_t)\nu(x_t^d)
    +\kappa_t\indc{x_t^d=x_1^d}.
\end{equation*}
Here $\indc{\cdot}$ denotes an indicator. For the reference endpoint law $\refdist=\nu^{\otimes D}$, the path is stationary and its coordinate posterior is available analytically:
\begin{equation}
    p_{\mathrm{ref},1\mid t}^{d}(z\mid x_t)
    =(1-\kappa_t)\nu(z)
      +\kappa_t\indc{z=x_t^d}.
    \label{eq:analytic-reference-posterior}
\end{equation}
This posterior has full categorical support for $t<1$ and serves as the fixed
posterior base of IEDG. We write
$a_t:=\dot\kappa_t/(1-\kappa_t)$ for the corresponding rate factor.

\subsection{Posterior-exact guidance and Bregman learning}

Let $p_1$ and $q_1$ be source and target terminal distributions with
$\supp(q_1)\subseteq\supp(p_1)$, where $\supp$ denotes support, and suppose that
they share the same conditional
path, $p_{t\mid 1}=q_{t\mid 1}$. If $r(x)=q_1(x)/p_1(x)$, then
the posterior guidance identity of DGM \citep{wan2026discrete} gives
\begin{align}
    q_{1\mid t}^{d}(z\mid x_t)
    &=\frac{
      p_{1\mid t}^{d}(z\mid x_t)h_t^d(z,x_t)
    }{
      \sum_{a\in\alphabet}
      p_{1\mid t}^{d}(a\mid x_t)h_t^d(a,x_t)
    }.
    \label{eq:dgm-identity}
\end{align}
where $h_t^d(z,x_t):=\mathbb E[
      r(X_1)\mid X_1^d=z,X_t=x_t]$.
The expectation is taken under $X_1\sim p_1$ and
$X_t\sim p_{t\mid 1}(\cdot\mid X_1)$.

DGM uses the positive Bregman loss
$\ell_{\mathrm{DGM}}(h,r):=h-r\log h$ for $h,r>0$.
For each $(t,d,z,x_t)$, its population optimum satisfies
\begin{equation}
\begin{aligned}
    \operatorname*{arg\,min}_{h>0}\;
    \mathbb E\!\left[
      \ell_{\mathrm{DGM}}\!\left(h,r(X_1)\right)
      \,\middle|\,
      X_1^d=z,X_t=x_t
    \right]
    &=
    \mathbb E\!\left[
      r(X_1)\mid X_1^d=z,X_t=x_t
    \right].
\end{aligned}
\label{eq:positive-bregman-losses}
\end{equation}
Hence, the population minimizer in
\eqref{eq:positive-bregman-losses} is exactly the DGM guidance term
$h_t^d(z,x_t)$ defined above. IEDG exploits the same conditional-mean
property, with the stagewise response introduced in
Section~\ref{sec:same-base-bridge}.
\section{Iterative Exact Discrete Guidance}
\label{sec:method}

Direct application of \eqref{eq:dgm-identity} learns the entire
reference-to-target correction in one step, even when the target has
poor overlap with the reference. IEDG replaces this global correction with local Boltzmann tilts. Each ideal stage applies the exact DGM update to its current source, while IEDG
expresses every updated posterior relative to the same analytic reference posterior $p_{\mathrm{ref},1\mid t}$, allowing one guide to represent all corrections learned so far.

\begin{figure*}[t]
    \centering
    \includegraphics[width=\textwidth,height=0.18\textheight,keepaspectratio]
    {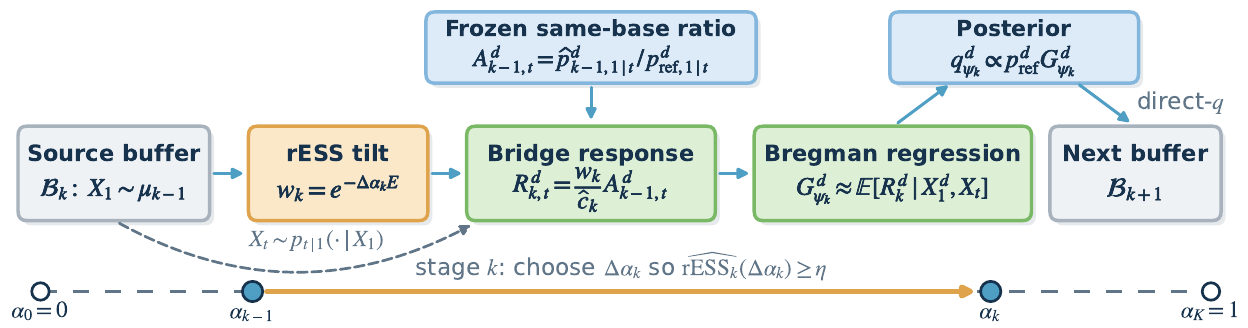}
    \vspace{-0.45em}
    \caption{\textbf{IEDG training.} Samples from the current buffer $\mathcal B_k$ estimate rESS and set the next
annealing increment $\Delta\alpha_k$. The energy weight $w_k$ and frozen
posterior ratio $A_{k-1,t}^d$ form the regression response $R_{k,t}^d$, whose
conditional mean is learned by positive Bregman regression as the accumulated
guide $G_{\psi_k}$. Combined with the analytic reference posterior, this defines
$q_{\psi_k,1\mid t}^d$, whose direct-$q$ rollout produces the next buffer
$\mathcal B_{k+1}$.}
    \label{fig:iedg-training}
    \vspace{-0.55em}
\end{figure*}

\subsection{Ideal annealing and the stage-local target}
\label{sec:stage-local-target}

Choose $0=\alpha_0<\alpha_1<\cdots<\alpha_K=1$. Define the continuous ideal
path $\{\pi_\alpha\}_{\alpha\in[0,1]}$, its stage values
$\pi_k:=\pi_{\alpha_k}$, and the tractable incremental weight by
\begin{align}
    \pi_\alpha(x)
    &:=\frac{\refdist(x)\exp\{-\alpha E(x)\}}
    {\sum_{x'\in\state}\refdist(x')\exp\{-\alpha E(x')\}},
    \label{eq:ideal-annealed-family}
\end{align}
Thus $\pi_0=\refdist$, $\pi_K=\targetdist$, and $\pi_k=\tiltop_{w_k}(\pi_{k-1})$.

At a non-ideal stage, let $\mu_{k-1}$ denote the stage-output law actually
promoted as the next source. We define the stage-local target as its exact Boltzmann tilt:
\begin{equation}
    \widetilde\pi_k(x)
    :=\frac{\mu_{k-1}(x)w_k(x)}
    {\sum_{x'\in\state}\mu_{k-1}(x')w_k(x')}.
    \label{eq:stage-local-target}
\end{equation}
where $w_k(x):=\exp\{-\Delta\alpha_kE(x)\}$ and $\Delta\alpha_k:=\alpha_k-\alpha_{k-1}$. 

\subsection{Stage-local posterior correction}
\label{sec:stage-teacher}

At stage $k$, the available endpoint law is $\mu_{k-1}$ and the local target is
the Boltzmann tilt $\widetilde\pi_k$ in \eqref{eq:stage-local-target}. Define
$Z_k:=\sum_{x\in\state}\mu_{k-1}(x)w_k(x)$. Then, for
$x\in\supp(\mu_{k-1})$,
$\widetilde\pi_k(x)/\mu_{k-1}(x)=w_k(x)/Z_k$, where $Z_k$ is independent of
$x$.

Under the shared forward path $X_1\sim\mu_{k-1}$ and
$X_t\sim p_{t\mid1}(\cdot\mid X_1)$, let
$p_{k-1,1\mid t}^{d}$ and $q_{k,1\mid t}^{d}$ denote the coordinate
posteriors induced by $\mu_{k-1}$ and $\widetilde\pi_k$, respectively. The
conditional density ratio in the DGM identity satisfies
\[
\mathbb E_{\mu_{k-1}}\!\left[
  \frac{\widetilde\pi_k(X_1)}{\mu_{k-1}(X_1)}
  \,\middle|\, X_1^d=z,X_t=x_t
\right]
=
\frac{1}{Z_k}
\mathbb E_{\mu_{k-1}}\!\left[
  w_k(X_1)\mid X_1^d=z,X_t=x_t
\right].
\]
Since $Z_k$ is common to all categories $z$, it cancels under posterior
normalization. Hence we define the stage guidance function
$h_{k,t}^{d}(z,x_t):=
\mathbb E_{\mu_{k-1}}[w_k(X_1)\mid X_1^d=z,X_t=x_t]$.
Substituting it into \eqref{eq:dgm-identity} gives
\begin{equation}
q_{k,1\mid t}^{d}(z\mid x_t)
=
\frac{
  p_{k-1,1\mid t}^{d}(z\mid x_t)
  h_{k,t}^{d}(z,x_t)
}{
  \sum_{a\in\alphabet}
  p_{k-1,1\mid t}^{d}(a\mid x_t)
  h_{k,t}^{d}(a,x_t)
}.
\label{eq:stage-teacher-posterior}
\end{equation}
Thus the exact stage posterior can be learned without evaluating $Z_k$.

\subsection{Same-base accumulated bridge}
\label{sec:same-base-bridge}

Equation~\ref{eq:stage-teacher-posterior} shows that stage $k$ reweights the
current source posterior by $h_{k,t}^{d}$. Because this posterior changes after
each rollout, the stagewise functions cannot simply be multiplied. IEDG instead
represents every posterior correction relative to the same analytic reference
posterior.

At the beginning of stage \(k\), freeze the offline source buffer
\(\mathcal B_k=\{X^{(i)}\}_{i=1}^{n_k}\), drawn from \(\mu_{k-1}\), together
with the preceding guide \(G_{\psi_{k-1}}\). The guide represents the posterior
\begin{equation*}
\widehat p_{k-1,1\mid t}^{d}(z\mid x_t)
:=
\frac{
p_{\mathrm{ref},1\mid t}^{d}(z\mid x_t)
G_{\psi_{k-1},t}^{d}(z,x_t)
}{
\sum_{a\in\alphabet}
p_{\mathrm{ref},1\mid t}^{d}(a\mid x_t)
G_{\psi_{k-1},t}^{d}(a,x_t)
}.
\end{equation*}
In particular, $G_{\psi_0,t}^{d}\equiv1$. Define the frozen posterior ratio
$A_{k-1,t}^{d}(z,x_t)
\!:=\!
\widehat p_{k-1,1\mid t}^{d}(z\!\mid\!x_t)/
p_{\mathrm{ref},1\mid t}^{d}(z\!\mid\!x_t)$.
IEDG uses this frozen posterior in place of the exact source posterior. Since
\(\widehat p_{k-1,1\mid t}^{d}
=p_{\mathrm{ref},1\mid t}^{d}A_{k-1,t}^{d}\), its category-wise update
\(\widehat p_{k-1,1\mid t}^{d}h_{k,t}^{d}\) is represented relative to the
analytic base by the target
\begin{equation}
G_{k,t}^{\star,d}(z,x_t)
:=
\frac{
A_{k-1,t}^{d}(z,x_t)h_{k,t}^{d}(z,x_t)
}{
\widehat c_k
}.
\label{eq:bridge-population-optimum}
\end{equation}
Here $\widehat c_k>0$ is fixed within the stage, improves numerical conditioning,
and cancels under posterior normalization; Appendix~\ref{app:fixed-stage-normalizer}
gives its buffer-based construction. To estimate \eqref{eq:bridge-population-optimum}
from source samples, IEDG uses the positive regression response
\begin{equation}
R_{k,t}^{d}(X_1,X_t)
:=
\frac{
w_k(X_1)A_{k-1,t}^{d}(X_1^d,X_t)
}{
\widehat c_k
}.
\label{eq:bridge-target}
\end{equation}
For fixed $(z,x_t)$, $A_{k-1,t}^{d}(z,x_t)$ is constant under the
conditional expectation, so
$\mathbb E[
R_{k,t}^{d}(X_1,X_t)\mid X_1^d=z,X_t=x_t]
=
A_{k-1,t}^{d}(z,x_t)h_{k,t}^{d}(z,x_t)/\widehat c_k
=
G_{k,t}^{\star,d}(z,x_t)$.
Thus $G_k^\star$ is the population regression target induced by the frozen
posterior. If
$\widehat p_{k-1,1\mid t}^{d}=p_{k-1,1\mid t}^{d}$, its category-wise update
is the exact numerator in \eqref{eq:stage-teacher-posterior}.

IEDG therefore fits a positive guide $G_{\psi_k}$ by minimizing
\begin{equation}
\mathcal L_k[G_\psi]
:=
\mathbb E_{\substack{
t\sim\tau,\;X_1\sim\mu_{k-1},\\
X_t\sim p_{t\mid1}(\cdot\mid X_1)}}
\left[
\frac{1}{D}\sum_{d=1}^{D}
\ell_{\mathrm{DGM}}\!\left(
G_{\psi,t}^{d}(X_1^d,X_t),
R_{k,t}^{d}(X_1,X_t)
\right)
\right].
\label{eq:bridge-objective}
\end{equation}

After fitting, the accumulated guide induces the fixed-base posterior
\begin{equation}
q_{\psi_k,1\mid t}^{d}(z\mid x_t)
=
\frac{
p_{\mathrm{ref},1\mid t}^{d}(z\mid x_t)
G_{\psi_k,t}^{d}(z,x_t)
}{
\sum_{a\in\alphabet}
p_{\mathrm{ref},1\mid t}^{d}(a\mid x_t)
G_{\psi_k,t}^{d}(a,x_t)
}.
\label{eq:fixed-base-parameterization}
\end{equation}
This posterior is frozen as
$\widehat p_{k,1\mid t}^{d}:=q_{\psi_k,1\mid t}^{d}$
for the next stage.

\begin{proposition}[Exact same-base bridge]
\label{thm:stage-exactness}
Under the support conditions of Appendix~\ref{app:proofs}, suppose that
\(\widehat p_{k-1,1\mid t}^{d}=p_{k-1,1\mid t}^{d}\) for every $d$,
$\tau$-almost every $t$, and $p_{k-1,t}$-almost every $x_t$.
For fixed $\widehat c_k$, $G_k^\star$ is the unique positive measurable
population minimizer of $\mathcal L_k[G]$, up to null sets under the training law,
and the fixed-base posterior induced by \(G_{k,t}^\star\) in
\eqref{eq:fixed-base-parameterization} equals the exact stage posterior
\(q_{k,1\mid t}^{d}\) in \eqref{eq:stage-teacher-posterior}.
\end{proposition}

\paragraph{Stagewise learning.}
Starting with \(G_{\psi_0}\equiv1\) and
\(\mathcal B_1\sim\refdist\), stage \(k\) fits \(G_{\psi_k}\) after replacing
$X_1\sim\mu_{k-1}$ in \eqref{eq:bridge-objective} with
$X_1\sim\operatorname{Unif}(\mathcal B_k)$. It then freezes the fitted guide
and uses direct-$q$ sampling to generate $\mathcal B_{k+1}$, whose endpoint law
$\mu_k$ becomes the next source. Algorithm~\ref{alg:iedg-training} gives the
complete procedure.

\subsection{ESS-controlled thermodynamic path}
\label{sec:ess-path}

Section~\ref{sec:stage-local-target} specifies the annealed family but not the
stage locations $\{\alpha_k\}$. Because a fixed temperature increment can have
very different overlap along the path, IEDG selects each increment using the
current source $\mu_{k-1}$. For $\Delta\geq0$, let
$w_\Delta(x):=e^{-\Delta E(x)}$. Its population relative effective sample size is
\begin{equation}
    \ress_k(\Delta)
    :=\frac{\bigl(\mathbb E_{\mu_{k-1}}[w_\Delta(X)]\bigr)^2}
    {\mathbb E_{\mu_{k-1}}[w_\Delta(X)^2]}.
    \label{eq:population-ess-ratio}
\end{equation}
ESS-based adaptation is standard in annealed SMC
\citep{zhou2016automatic,syed2026optimized}, and the local geometry of a Gibbs
path is characterized by thermodynamic length \citep{crooks2007measuring}.
The following proposition specializes these relations to IEDG's stagewise tilt.
\begin{proposition}[R\'enyi-controlled thermodynamic increments]
\label{prop:ess-geometry}
For every $\Delta\geq0$, with
$D_2(P\Vert Q):= \log\sum_{x:Q(x)>0}P(x)^2/Q(x)$,
the population rESS satisfies
\begin{equation}
    \ress_k(\Delta)
    =\exp\!\left\{-D_2\!\left(
      \tiltop_{w_\Delta}(\mu_{k-1})\Vert\mu_{k-1}
    \right)\right\}.
    \label{eq:ess-renyi-identity}
\end{equation}
If $\mu_{k-1}=\pi_\alpha$ lies on the ideal Boltzmann path, then
\begin{equation}
    -\log\ress_\alpha(\Delta)
    =\Delta^2\operatorname{Var}_{\pi_\alpha}[E(X)]
      +O(\Delta^3)
    \qquad\text{as }\Delta\to0.
    \label{eq:ess-local-expansion}
\end{equation}
\end{proposition}

Thus, for $\eta\in(0,1]$, $\ress_k(\Delta)\geq\eta$ imposes a finite-step R\'enyi-$2$ radius
$-\log\eta$ and, locally on the ideal path, approximates equal
thermodynamic-length increments. Given the source buffer
$\mathcal B_k=\{X^{(i)}\}_{i=1}^{n_k}$, IEDG uses the plug-in rule
\begin{equation}
\widehat{\ress}_k(\Delta)
:=
\frac{
  \left(\sum_{x\in\mathcal B_k}e^{-\Delta E(x)}\right)^2
}{
  n_k\sum_{x\in\mathcal B_k}e^{-2\Delta E(x)}
},
\qquad
\Delta_k^{\mathrm{ESS}}
:=
\max\left\{
  0\leq\Delta\leq1-\alpha_{k-1}:
  \widehat{\ress}_k(\Delta)\geq\eta
\right\}.
\label{eq:ess-path-rule}
\end{equation}
In development, we hold the groupwise rESS thresholds fixed across stages and
use the resulting empirical rule to construct a path. Its endpoints are then
frozen across formal runs; Appendix~\ref{app:frozen-ess-path} gives the
protocol and realized paths.

\subsection{Sampling with the accumulated guide}
\label{sec:reverse-sampling}
Using the fixed-base posterior in
\eqref{eq:fixed-base-parameterization}, the final guide defines the
posterior-marginal CTMC rates
\begin{equation*}
u_{\psi_K,t}^{d}(z,x)
=
a_t q_{\psi_K,1\mid t}^{d}(z\mid x)\indc{z\neq x^d}.
\end{equation*}
where $a_t=\dot\kappa_t/(1-\kappa_t)$. Starting from
$X_0\sim\refdist$, we simulate these rates with the direct-$q$ tau-leaping
scheme in Algorithm~\ref{alg:direct-q-sampling}. Final inference evaluates only
$p_{\mathrm{ref},1\mid t}$ and $G_{\psi_K}$, independent of the number of
training stages.

\section{Learning Error and Stagewise Stability}
\label{sec:theory}

We analyze a finite horizon $T<1$ to avoid the endpoint singularity. At stage
$k$, substituting the frozen posterior $\widehat p_{k-1,1\mid t}^{d}$ for the
exact source posterior $p_{k-1,1\mid t}^{d}$ in
\eqref{eq:stage-teacher-posterior} gives $\bar q_{k,1\mid t}^{d}$ in
\eqref{eq:operational-stage-posterior}. This separates frozen-posterior
mismatch between $q_{k,1\mid t}^{d}$ and $\bar q_{k,1\mid t}^{d}$ from
guide-fitting error between $\bar q_{k,1\mid t}^{d}$ and
$q_{\psi_k,1\mid t}^{d}$. The former enters the stagewise remainder below; the
latter is controlled by Bregman excess risk.

\paragraph{From Bregman risk to one-stage sampling error.}
\label{sec:bregman-calibration}
Let \(G_k^\star\) be the population minimizer in
\eqref{eq:bridge-population-optimum}, and let
$\Delta\mathcal L_k
:=\mathcal L_k[G_{\psi_k}]-\mathcal L_k[G_k^\star]$
denote the population excess risk of the stagewise Bregman objective. For training-time
density $\tau$, define
\begin{equation*}
    \mathcal A_{\tau,T}^2:=\int_0^T\frac{a_t^2}{\tau(t)}\dd t,
    \qquad
    \mathsf K_{k,T}:=D\mathcal A_{\tau,T}\sqrt{2\mathcal C_{k,T}},
\end{equation*}
where $\mathcal C_{k,T}$ is the conditional coverage ratio defined in
Appendix~\ref{app:one-stage-calibration}. Let
$\mu_{k,t}^{\mathrm{ct}}$ denote the learned continuous-time law,
$\mu_k^{\mathrm{ct}}$ its terminal law, and $q_{k,T}$ the exact stage law at
time $T$. The terms not controlled by $\Delta\mathcal L_k$ are collected in
\begin{equation*}
r_{k,T}^{\mathrm{pr}}
:=
\underbrace{2\delta_{k,T}^{\mathrm{src}}}_{\text{frozen-posterior mismatch}}
+
\underbrace{\tv(\mu_k,\mu_k^{\mathrm{ct}})}_{\text{numerical simulation}}
+
\underbrace{\tv(\mu_k^{\mathrm{ct}},\mu_{k,T}^{\mathrm{ct}})}_{
\text{learned-process truncation}}
+
\underbrace{\tv(q_{k,T},\widetilde\pi_k)}_{
\text{exact-process truncation}}.
\end{equation*}
Here \(\tv\) denotes total variation and $\delta_{k,T}^{\mathrm{src}}$ is the
rate-weighted discrepancy between the frozen posterior and that induced by the
current source. Thus $r_{k,T}^{\mathrm{pr}}$ contains exactly the
frozen-posterior, numerical-simulation, and two finite-$T$ truncation terms.

\begin{proposition}[One-stage Bregman-to-sampling bound]
\label{prop:bregman-posterior-calibration}
Assume the source, fixed-base, and frozen coordinate posteriors have full
categorical support on the training contexts, $\tau(t)>0$ on $[0,T]$, and
$\mathcal A_{\tau,T},\mathcal C_{k,T}<\infty$. Then
\begin{equation}
    \tv(\mu_k,\widetilde\pi_k)
    \leq
    \frac{\mathsf K_{k,T}}
    {\sqrt{\ress_k(\Delta\alpha_k)}}
    \sqrt{\Delta\mathcal L_k}
    +r_{k,T}^{\mathrm{pr}}.
    \label{eq:one-stage-practical-bound}
\end{equation}
\end{proposition}

The fitting contribution scales as
$\sqrt{\Delta\mathcal L_k/\ress_k(\Delta\alpha_k)}$. Hence small Bregman excess risk is not sufficient: the error is amplified when consecutive
distributions overlap poorly. The tail is explicit,
$\tv(q_{k,T},\widetilde\pi_k)\leq1-\kappa_T^D$; Appendix~\ref{app:one-stage-calibration}
gives the remaining definitions and finite-time controls.

\paragraph{Stagewise propagation.}
\label{sec:terminal-stability}
Positive reweighting propagates each local discrepancy through the
Boltzmann path. Let $e_k:=\tv(\mu_k,\pi_k)$ and define the stage-local,
normalizer-aware sensitivity
\begin{equation*}
\Lambda_k
:=
\frac{\max_{x\in\state}w_k(x)}
{
\max\!\left\{
\mathbb E_{\mu_{k-1}}[w_k(X)],
\mathbb E_{\pi_{k-1}}[w_k(X)]
\right\}
}.
\end{equation*}

\begin{theorem}[Stagewise stability recursion]
\label{thm:posterior-terminal-stability}
Assume the support and finite-coverage conditions of
Proposition~\ref{prop:bregman-posterior-calibration} hold at every stage,
$\ress_k(\Delta\alpha_k)\geq\eta_k$, and $\mu_0=\pi_0=\refdist$. Then, for
$k=1,\ldots,K$,
\begin{equation}
    e_k
    \leq
      \frac{\mathsf K_{k,T}}{\sqrt{\eta_k}}
      \sqrt{\Delta\mathcal L_k}
      +r_{k,T}^{\mathrm{pr}}
      +\Lambda_k e_{k-1}.
    \label{eq:recursive-error-bound}
\end{equation}
\end{theorem}

The recursion separates current-stage from inherited error:
$\eta_k^{-1/2}$ is the overlap penalty on the guide-fitting term,
$r_{k,T}^{\mathrm{pr}}$ collects the remaining local discrepancies, and
$\Lambda_k$ propagates $e_{k-1}$ through normalized Boltzmann reweighting. Iterating \eqref{eq:recursive-error-bound}
gives a terminal bound; an optional distribution-free envelope is deferred to
Appendix~\ref{app:multistage-propagation}.

\begin{corollary}[Exact recovery of the untruncated construction]
\label{cor:ideal-recovery}
If, at every stage,
$\widehat p_{k-1,1\mid t}^{d}=p_{k-1,1\mid t}^{d}$ for every $d$ and almost
every rate-evaluation context, $G_{\psi_k}=G_k^\star$ at those contexts, and
the exact posterior-marginal process is simulated to its terminal law, then
$\Delta\mathcal L_k=0$ and
$r_{k,T}^{\mathrm{pr}}\to0$ as $T\uparrow1$; consequently
$\mu_k=\pi_k$ for all $k$ and
$\mu_K=\targetdist$.
\end{corollary}

\section{Experiments}
\label{sec:experiments}

Our experiments ask whether iteration improves (i) recovery of an exactly enumerable target distribution, (ii) sampling across high-dimensional phase transitions,
and (iii) graph-conditioned optimization. We study enumerable Ising $4\times4$,
Ising and Potts $16\times16$, and Barab\'asi--Albert (BA) Max-Cut. We compare IEDG with one-shot DGM, our UDNS reproduction based on DASBS~\citep{guo2026dasbs},
MDNS~\citep{zhu2025mdns}, MetaDNS~\citep{du2026metadns}, and
PDNS~\citep{guo2026pdns}. Appendix~\ref{app:experiments} gives complete protocol, and Appendix~\ref{app:additional-experiments} provides ablations and diagnostics.

\subsection{Exact distribution recovery}
\label{sec:exp-ising4}

For spins $s_i\in\{-1,+1\}$ on a periodic $L\times L$ square lattice, the Ising target is
\begin{equation}
    \pi_\beta^{\mathrm I}(s)
    =
    \frac{1}{Z_\beta^{\mathrm I}}
    \exp\{-\beta H_{\mathrm I}(s)\},
    \text{ where }
    H_{\mathrm I}(s)
    =
    -J\sum_{\langle i,j\rangle}s_i s_j
    -h\sum_i s_i.
    \label{eq:ising-benchmark}
\end{equation}
$Z_\beta^{\mathrm I}$ is the partition function. For Ising $4\times4$
with $(J,h)=(1,0.1)$, enumeration permits exact TV,
$\mathrm{KL}(\widehat p\Vert\pi)$, and $\chi^2(\widehat p\Vert\pi)$ evaluation.
Table~\ref{tab:ising4-exact} reports the ordered endpoint
$\beta=0.6$ with equal-budget i.i.d. target samples as Exact MC; the full
temperature sweep is in Appendix~\ref{app:supplemental-diagnostics}.

\begin{table}[H]
    \centering
    \caption{Distribution-level errors on exactly enumerated Ising $4\times4$ at $\beta=0.6$ (lower is
    better). Exact MC is the equal-budget i.i.d. reference; $\dagger$ marks the
    MDNS-reported $F_{\rm WDCE}$ result. Bold marks the best neural sampler.}
    \label{tab:ising4-exact}
    \small
    \setlength{\tabcolsep}{7pt}
    \begin{tabular}{lccc}
        \toprule
        Method & $\mathrm{TV}$
        & $\mathrm{KL}(\widehat p\Vert\pi)$
        & $\chi^2(\widehat p\Vert\pi)$ \\
        \midrule
        Exact MC & $0.00411{\pm}0.00019$ & $0.00426{\pm}0.00006$
          & $0.0515{\pm}0.0122$ \\
        \midrule
        One-shot DGM & 0.779 & 2.55 & 722 \\
        UDNS (reproduced) & 0.150 & 0.165 & 0.237 \\
        MDNS ($F_{\rm WDCE}$)$^\dagger$ & 0.0418 & 0.0282 & 1.66 \\
        IEDG & \textbf{0.0314} & \textbf{0.00993} & \textbf{0.229} \\
        \bottomrule
    \end{tabular}
\end{table}

IEDG achieves the lowest TV, KL, and $\chi^2$ errors among the evaluated neural baselines, while a visible gap to the equal-budget Exact-MC floor remains.

\subsection{High-dimensional lattice models}
\label{sec:exp-lattices}

The Ising benchmark uses \eqref{eq:ising-benchmark} with $L=16$ and
$(J,h)=(1,0)$. For three-state labels $x_i\in\{1,2,3\}$ on a periodic $L\times L$ square lattice, the Potts target is
\[
    \pi_\beta^{\mathrm P}(x)
    =\frac{1}{Z_\beta^{\mathrm P}}
      \exp\{-\beta H_{\mathrm P}(x)\},
    \text{ where }
    H_{\mathrm P}(x)
    =-J\sum_{\langle i,j\rangle}\indc{x_i=x_j},
\]
$Z_\beta^{\mathrm P}$ is the partition function and $L=16$, $J=1$. We evaluate disordered,
near-critical, and ordered regimes at $\beta\in\{0.28,0.4407,0.6\}$ for Ising
and $\beta\in\{0.5,1.005,1.2\}$ for Potts
(Table~\ref{tab:lattice-main}). To stabilize the high-dimensional transports, we use two training components detailed in Appendix~\ref{app:implementation}: analytic preconditioning provides a short-range belief-propagation initialization for the learned residual, and conditional reweighting supplies lower-variance auxiliary
supervision. Both lattices use preconditioning, whereas conditional
reweighting ($\lambda_{\rm CR}=1$) is restricted to the transports ending at
Ising $\beta=0.4407$ and Potts $\beta=1.005$. Ablation of such components is introduced in Appendix~\ref{app:component-ablations}. 

We use independent Swendsen--Wang (SW) samples as the reference
distribution for both $16\times16$ lattice benchmarks. Relative to SW,
Table~\ref{tab:lattice-main} reports the local Mag. and MDNS Corr.
aggregates together with phase-sensitive errors: Ising $x_\uparrow$ JS
and Potts permutation-invariant dominant-mode $\ell_1$.
The phase-sensitive terms assess the relative masses of the two Ising
phases and three Potts phases; all metrics are minimized. Definitions
and additional diagnostics are given in
Appendices~\ref{app:mdns-metrics}--\ref{app:supplemental-diagnostics}.
Validation and test evaluation use independent SW pools and disjoint
rollout seeds.

\begin{figure*}[t]
    \centering
    \includegraphics[width=0.96\textwidth]{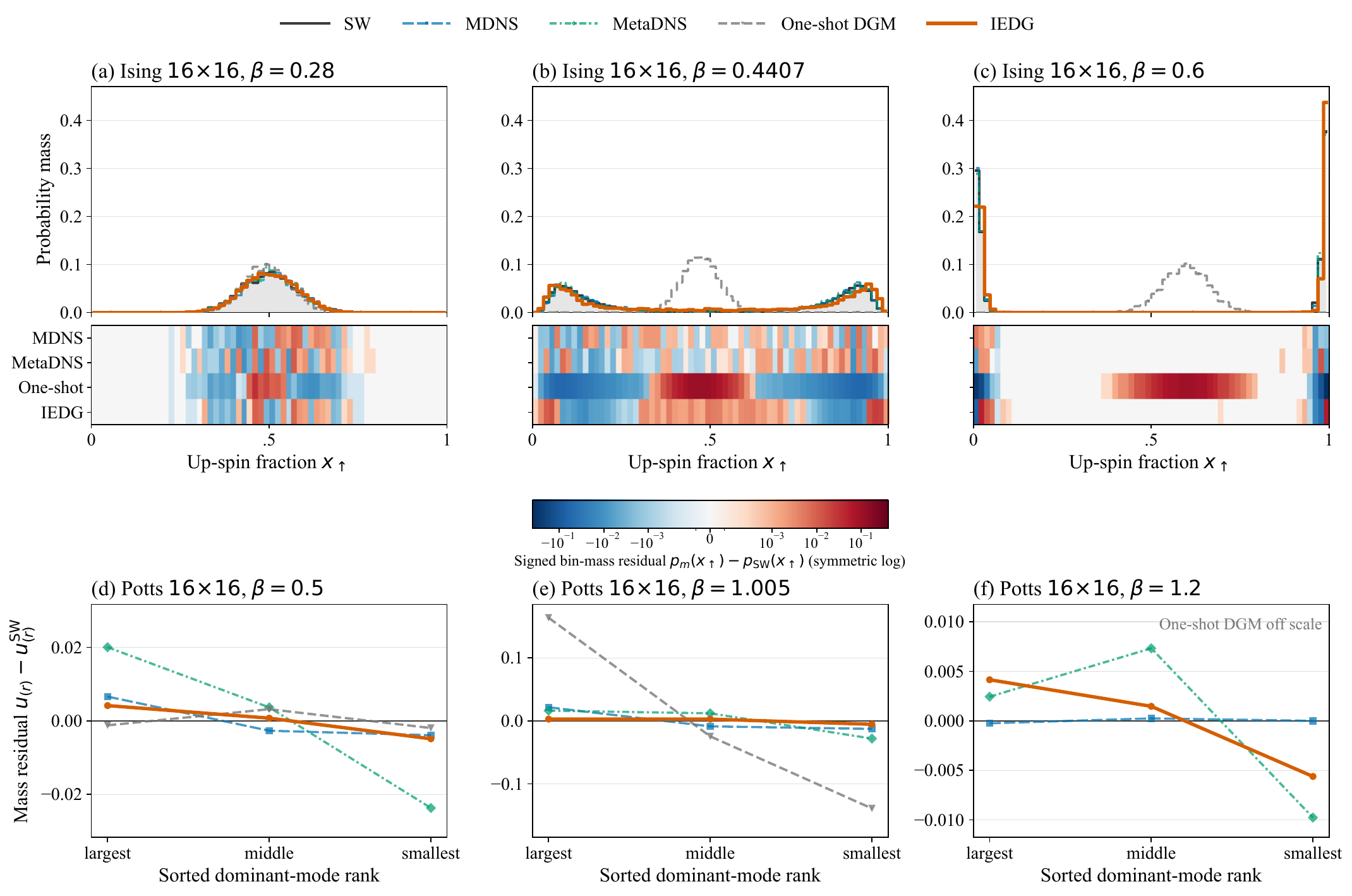}
    \caption{Phase coverage on $16\times16$ lattices: Ising $x_\uparrow$ histograms
    (a--c) and Potts sorted dominant-mode mass residuals
    relative to SW (d--f) across three regimes. In (a--c), the upper subpanels
    show the full distributions, while
    the maps below show signed per-bin residuals on a common symmetric-log
    scale (red: excess; blue: deficit). Panel (f) omits the off-scale one-shot
    curve.}
    \label{fig:phase-coverage}
\end{figure*}

\begin{table*}[t]
    \centering
    \caption{Sampling $16\times16$ lattice models across thermodynamic regimes
    (lower is better). Bold marks the best result.}
    \label{tab:lattice-main}
    \scriptsize
    \setlength{\tabcolsep}{2.5pt}
    \renewcommand{\arraystretch}{0.90}

    \textbf{(a) Ising $16\times16$}\par\vspace{0.2em}
    \begin{tabular}{lcccc|cccc|cccc}
        \toprule
        & \multicolumn{4}{c|}{$\beta=0.28$}
        & \multicolumn{4}{c|}{$\beta=0.4407$}
        & \multicolumn{4}{c}{$\beta=0.6$} \\
        Method
        & Mag. & Corr. agg. & $x_\uparrow$ JS
        & 
        & Mag. & Corr. agg. & $x_\uparrow$ JS
        &
        & Mag. & Corr. agg. & $x_\uparrow$ JS
        & \\
        \midrule
        Analytic preconditioner
        & 0.00615 & 1.51 & 0.0863
        & 
        & \textbf{0.00409} & 16.5 & 0.609
        &
        & 0.0169 & 27.1 & 0.693
        & \\

        MDNS (official ckpt.)
        & 0.0112 & \textbf{0.161} & 0.00937
        &
        & 0.0105 & 0.171 & \textbf{0.00980}
        &
        & 0.0173 & 0.0582 & \textbf{0.00336}
        & \\

        MetaDNS (official ckpt.)
        & 0.00839 & 0.195 & 0.00889
        &
        & 0.0317 & \textbf{0.124} & 0.0132
        &
        & \textbf{0.00490} & 0.0892 & 0.00574
        & \\

        One-shot DGM
        & 0.0112 & 0.402 & 0.0159
        &
        & 0.0587 & 16.3 & 0.599
        &
        & 0.170 & 26.8 & 0.693
        & \\

        IEDG
        & \textbf{0.00613} & 0.242 & \textbf{0.00742}
        &
        & 0.0106 & 0.150 & 0.0252
        &
        & 0.0157 & \textbf{0.0537} & 0.0193
        & \\
        \bottomrule
    \end{tabular}

    \vspace{0.3em}

    \textbf{(b) Three-state Potts $16\times16$}\par\vspace{0.2em}
    \begin{tabular}{lcccc|cccc|cccc}
        \toprule
        & \multicolumn{4}{c|}{$\beta=0.5$}
        & \multicolumn{4}{c|}{$\beta=1.005$}
        & \multicolumn{4}{c}{$\beta=1.2$} \\
        Method
        & Mag. & Corr. agg. & Mode $\ell_1$
        &
        & Mag. & Corr. agg. & Mode $\ell_1$
        &
        & Mag. & Corr. agg. & Mode $\ell_1$
        & \\
        \midrule
        Analytic preconditioner
        & 0.0294 & 0.152 & 0.0161
        &
        & 0.156 & 10.4 & 0.0372
        &
        & 0.0237 & 16.9 & 0.0254
        & \\

        MDNS (official ckpt.)
        & 0.0285 & 0.0675 & 0.0132
        &
        & 0.340 & 0.212 & 0.0431
        &
        & \textbf{0.0179} & 0.0386 & \textbf{0.000490}
        & \\

        MetaDNS (official ckpt.)
        & 0.246 & 0.114 & 0.0474
        &
        & 0.483 & \textbf{0.123} & 0.0563
        &
        & 0.197 & 0.131 & 0.0195
        & \\

        One-shot DGM
        & 0.0309 & 0.0677 & \textbf{0.00635}
        &
        & 0.488 & 10.9 & 0.328
        &
        & 0.947 & 17.7 & 0.825
        & \\

        IEDG
        & \textbf{0.0284} & \textbf{0.0673} & 0.00977
        &
        & \textbf{0.108} & \textbf{0.123} & \textbf{0.0109}
        &
        & 0.0757 & \textbf{0.0379} & 0.0112
        & \\
        \bottomrule
    \end{tabular}
\end{table*}

\paragraph{Ising: two-phase recovery.}
The $x_\uparrow$ law tracks the symmetry-breaking transition from a
single disordered mode through critical broadening to two ordered modes
(Figure~\ref{fig:phase-coverage}(a--c)). IEDG leads the learned samplers in
Mag. and $x_\uparrow$ JS in the disordered regime; near criticality, iteration
repairs the correlation and phase-coverage failures of one-shot DGM while
improving Corr. agg. over MDNS. At the ordered endpoint, IEDG attains the
lowest Corr. error and improves on MDNS in magnetization, although MDNS and MetaDNS recover the two phase weights more accurately.

\paragraph{Potts: three-phase recovery.}
Ranked dominant-color masses measure probability across the three
symmetry-related sectors (Figure~\ref{fig:phase-coverage}(d--f)). IEDG remains
competitive in the disordered regime and, near criticality, repairs one-shot
mode collapse, leading the learned samplers in Mag. and Mode $\ell_1$ while
matching the best Corr. agg. At the ordered endpoint, it sharply improves the
one-shot Mag. and Corr. errors while covering all sectors. MDNS
retains more accurate Mag. and phase weights.

\subsection{Graph-conditioned Max-Cut sampling}
\label{sec:exp-maxcut}

For an undirected BA graph $G=(V,E)$ with $n$ vertices and binary
partition $x$, the cut size is
$C_G(x):=\sum_{\{i,j\}\in E}\indc{x_i\ne x_j}$. We sample from
$\pi_\lambda(x\!\mid\!G)\propto e^{\lambda C_G(x)}$ at $\lambda=5$, which
concentrates probability on large cuts; Appendix~\ref{app:maxcut-protocol}
gives the complete target. For each
$n\in[20,32],[40,64],[100,128]$, we train on 1024 BA graphs and
evaluate 512 samples on each of 32 held-out graphs. Relative to the certified
optimum $C_G^\star:=\max_x C_G(x)$, $(R_{\max},R_{\rm avg})$ measure
best-of-budget and average sample quality.

\begin{table}[!htbp]
    \centering
    \caption{BA Max-Cut approximation ratios on 32 held-out graphs (higher is
    better), using 512 samples per graph. $\dagger$ marks PDNS
    values~\citep{guo2025pdnsv1}. Bold marks the best result.}
    \label{tab:maxcut-main}
    \small
    \setlength{\tabcolsep}{8.0pt}
    \renewcommand{\arraystretch}{0.92}
    \begin{tabular}{lcccccc}
        \toprule
        & \multicolumn{2}{c}{$n\in[20,32]$}
        & \multicolumn{2}{c}{$n\in[40,64]$}
        & \multicolumn{2}{c}{$n\in[100,128]$} \\
        \cmidrule(lr){2-3}\cmidrule(lr){4-5}\cmidrule(lr){6-7}
        Method & Max. & Avg. & Max. & Avg. & Max. & Avg. \\
        \midrule
        Uniform & 0.888 & 0.696 & 0.825 & 0.687 & 0.769 & 0.678 \\
        One-shot DGM & 0.962 & 0.854 & 0.884 & 0.798 & 0.814 & 0.748 \\
        PDNS$^\dagger$ & 0.963 & 0.911 & 0.897 & 0.869 & 0.876 & 0.840 \\
        IEDG & \textbf{1.00} & \textbf{0.965} & \textbf{0.987}
          & \textbf{0.920} & \textbf{0.915} & \textbf{0.871} \\
        \bottomrule
    \end{tabular}
\end{table}

Across all three size ranges, IEDG improves both ratios over one-shot DGM and
PDNS; in particular, it reaches the certified optimum within budget on the
smallest graphs.

\section{Conclusion}

IEDG replaces a global energy correction with ESS-controlled stagewise
posterior updates accumulated in one guide. When the frozen
posterior matches that induced by the current source, population-optimal
updates recover the intended reverse dynamics; otherwise, our analysis links
Bregman fitting error, overlap, and stagewise stability. Experiments demonstrate
improved full-distribution recovery, substantial gains over one-shot guidance on large
lattices, and stronger BA Max-Cut ratios than other baselines.

rESS controls endpoint overlap but not conditional-category coverage or rare
phases, so posterior errors may propagate across stages. Stronger variance
reduction and online joint updates of the guide and source distribution are
promising directions.

\bibliography{references}

\begin{thebibliography}{40}
\providecommand{\natexlab}[1]{#1}
\providecommand{\url}[1]{\texttt{#1}}
\expandafter\ifx\csname urlstyle\endcsname\relax
  \providecommand{\doi}[1]{doi: #1}\else
  \providecommand{\doi}{doi: \begingroup \urlstyle{rm}\Url}\fi

\bibitem[Akhound-Sadegh et~al.(2024)Akhound-Sadegh, Rector-Brooks, Bose,
  Mittal, Lemos, Liu, Sendera, Ravanbakhsh, Gidel, Bengio, Malkin, and
  Tong]{akhoundsadegh2024iterated}
Tara Akhound-Sadegh, Jarrid Rector-Brooks, Joey Bose, Sarthak Mittal, Pablo
  Lemos, Cheng-Hao Liu, Marcin Sendera, Siamak Ravanbakhsh, Gauthier Gidel,
  Yoshua Bengio, Nikolay Malkin, and Alexander Tong.
\newblock Iterated denoising energy matching for sampling from {Boltzmann}
  densities.
\newblock In \emph{Proceedings of the 41st International Conference on Machine
  Learning}, volume 235 of \emph{Proceedings of Machine Learning Research},
  pp.\  760--786, 2024.

\bibitem[Arbel et~al.(2021)Arbel, Matthews, and Doucet]{arbel2021annealed}
Michael Arbel, Alexander G. D.~G. Matthews, and Arnaud Doucet.
\newblock Annealed flow transport monte carlo.
\newblock In \emph{International Conference on Machine Learning}, 2021.

\bibitem[Austin et~al.(2021)Austin, Johnson, Ho, Tarlow, and van~den
  Berg]{austin2021structured}
Jacob Austin, Daniel~D. Johnson, Jonathan Ho, Daniel Tarlow, and Rianne van~den
  Berg.
\newblock Structured denoising diffusion models in discrete state-spaces.
\newblock In \emph{Advances in Neural Information Processing Systems},
  volume~34, 2021.

\bibitem[Banerjee et~al.(2005)Banerjee, Guo, and Wang]{banerjee2005optimality}
Arindam Banerjee, Xin Guo, and Hui Wang.
\newblock On the optimality of conditional expectation as a bregman predictor.
\newblock \emph{IEEE Transactions on Information Theory}, 51\penalty0
  (7):\penalty0 2664--2669, 2005.

\bibitem[Bengio et~al.(2023)Bengio, Lahlou, Deleu, Hu, Tiwari, and
  Bengio]{bengio2023gflownet}
Yoshua Bengio, Salem Lahlou, Tristan Deleu, Edward~J. Hu, Mo~Tiwari, and
  Emmanuel Bengio.
\newblock {GFlowNet} foundations.
\newblock \emph{Journal of Machine Learning Research}, 24\penalty0
  (210):\penalty0 1--55, 2023.

\bibitem[Beskos et~al.(2016)Beskos, Jasra, Kantas, and
  Thiery]{beskos2016convergence}
Alexandros Beskos, Ajay Jasra, Nikolas Kantas, and Alexandre Thiery.
\newblock On the convergence of adaptive sequential {Monte Carlo} methods.
\newblock \emph{The Annals of Applied Probability}, 26\penalty0 (2):\penalty0
  1111--1146, 2016.
\newblock \doi{10.1214/15-AAP1113}.

\bibitem[Campbell et~al.(2022)Campbell, Benton, De~Bortoli, Rainforth,
  Deligiannidis, and Doucet]{campbell2022continuous}
Andrew Campbell, Joe Benton, Valentin De~Bortoli, Tom Rainforth, George
  Deligiannidis, and Arnaud Doucet.
\newblock A continuous time framework for discrete denoising models.
\newblock In \emph{Advances in Neural Information Processing Systems},
  volume~35, 2022.

\bibitem[Campbell et~al.(2024)Campbell, Yim, Barzilay, Rainforth, and
  Jaakkola]{campbell2024generative}
Andrew Campbell, Jason Yim, Regina Barzilay, Tom Rainforth, and Tommi Jaakkola.
\newblock Generative flows on discrete state-spaces: Enabling multimodal flows
  with applications to protein co-design.
\newblock In \emph{International Conference on Machine Learning}, 2024.

\bibitem[Carter et~al.(2026)Carter, Choi, Tamogashev, Elvira, and
  Malkin]{carter2026offpolicy}
Arran Carter, Sanghyeok Choi, Kirill Tamogashev, V{\'i}ctor Elvira, and Nikolay
  Malkin.
\newblock Discrete diffusion samplers and bridges: Off-policy algorithms and
  applications in latent spaces.
\newblock In \emph{International Conference on Machine Learning}, 2026.

\bibitem[Crooks(2007)]{crooks2007measuring}
Gavin~E. Crooks.
\newblock Measuring thermodynamic length.
\newblock \emph{Physical Review Letters}, 99\penalty0 (10):\penalty0 100602,
  2007.
\newblock \doi{10.1103/PhysRevLett.99.100602}.

\bibitem[Del~Moral et~al.(2006)Del~Moral, Doucet, and
  Jasra]{delmoral2006sequential}
Pierre Del~Moral, Arnaud Doucet, and Ajay Jasra.
\newblock Sequential monte carlo samplers.
\newblock \emph{Journal of the Royal Statistical Society: Series B (Statistical
  Methodology)}, 68\penalty0 (3):\penalty0 411--436, 2006.

\bibitem[Du et~al.(2026)Du, Nam, Choi, Guo, Edamadaka, Sha, Pan, Chen, Tao, and
  G{\'o}mez-Bombarelli]{du2026metadns}
Xiaochen Du, Juno Nam, Jaemoo Choi, Wei Guo, Sathya Edamadaka, Junyi Sha, Elton
  Pan, Yongxin Chen, Molei Tao, and Rafael G{\'o}mez-Bombarelli.
\newblock {MetaDNS}: Enhancing exploration in discrete neural samplers via
  metadynamics.
\newblock In \emph{Proceedings of the 43rd International Conference on Machine
  Learning}, volume 306 of \emph{Proceedings of Machine Learning Research},
  2026.

\bibitem[Gat et~al.(2024)Gat, Remez, Shaul, Kreuk, Chen, Synnaeve, Adi, and
  Lipman]{gat2024discrete}
Itai Gat, Tal Remez, Neta Shaul, Felix Kreuk, Ricky T.~Q. Chen, Gabriel
  Synnaeve, Yossi Adi, and Yaron Lipman.
\newblock Discrete flow matching.
\newblock In \emph{Advances in Neural Information Processing Systems},
  volume~37, pp.\  133345--133385, 2024.

\bibitem[Grathwohl et~al.(2021)Grathwohl, Swersky, Hashemi, Duvenaud, and
  Maddison]{grathwohl2021oops}
Will Grathwohl, Kevin Swersky, Milad Hashemi, David Duvenaud, and Chris
  Maddison.
\newblock Oops i took a gradient: Scalable sampling for discrete distributions.
\newblock In \emph{Proceedings of the 38th International Conference on Machine
  Learning}, volume 139 of \emph{Proceedings of Machine Learning Research},
  pp.\  3831--3841, 2021.

\bibitem[Guo et~al.(2025)Guo, Choi, Zhu, Tao, and Chen]{guo2025pdnsv1}
Wei Guo, Jaemoo Choi, Yuchen Zhu, Molei Tao, and Yongxin Chen.
\newblock Proximal diffusion neural sampler.
\newblock \emph{arXiv preprint arXiv:2510.03824v1}, 2025.
\newblock URL \url{https://arxiv.org/abs/2510.03824v1}.

\bibitem[Guo et~al.(2026{\natexlab{a}})Guo, Choi, Zhu, Tao, and
  Chen]{guo2026pdns}
Wei Guo, Jaemoo Choi, Yuchen Zhu, Molei Tao, and Yongxin Chen.
\newblock Proximal diffusion neural sampler.
\newblock In \emph{International Conference on Learning Representations},
  2026{\natexlab{a}}.

\bibitem[Guo et~al.(2026{\natexlab{b}})Guo, Zhu, Du, Nam, Chen,
  G{\'o}mez-Bombarelli, Liu, Tao, and Choi]{guo2026dasbs}
Wei Guo, Yuchen Zhu, Xiaochen Du, Juno Nam, Yongxin Chen, Rafael
  G{\'o}mez-Bombarelli, Guan-Horng Liu, Molei Tao, and Jaemoo Choi.
\newblock Discrete adjoint schr\"odinger bridge sampler.
\newblock In \emph{International Conference on Machine Learning},
  2026{\natexlab{b}}.
\newblock URL \url{https://openreview.net/forum?id=G9KydTWzZL}.

\bibitem[He et~al.(2026)He, Rojas, and Tao]{he2026guidance}
Ye~He, Kevin Rojas, and Molei Tao.
\newblock What exactly does guidance do in masked discrete diffusion models?
\newblock In \emph{International Conference on Learning Representations}, 2026.

\bibitem[Holderrieth et~al.(2025)Holderrieth, Albergo, and
  Jaakkola]{holderrieth2025leaps}
Peter Holderrieth, Michael~Samuel Albergo, and Tommi Jaakkola.
\newblock {LEAPS}: A discrete neural sampler via locally equivariant networks.
\newblock In \emph{International Conference on Machine Learning}, 2025.

\bibitem[Hoogeboom et~al.(2021)Hoogeboom, Nielsen, Jaini, Forr{\'e}, and
  Welling]{hoogeboom2021argmax}
Emiel Hoogeboom, Didrik Nielsen, Priyank Jaini, Patrick Forr{\'e}, and Max
  Welling.
\newblock Argmax flows and multinomial diffusion: Learning categorical
  distributions.
\newblock In \emph{Advances in Neural Information Processing Systems},
  volume~34, pp.\  12454--12465, 2021.

\bibitem[Jasra et~al.(2011)Jasra, Stephens, Doucet, and
  Tsagaris]{jasra2011inference}
Ajay Jasra, David~A. Stephens, Arnaud Doucet, and Theodoros Tsagaris.
\newblock Inference for {L}{\'e}vy-driven stochastic volatility models via
  adaptive sequential {Monte Carlo}.
\newblock \emph{Scandinavian Journal of Statistics}, 38\penalty0 (1):\penalty0
  1--22, 2011.
\newblock \doi{10.1111/j.1467-9469.2010.00723.x}.

\bibitem[Lou et~al.(2024)Lou, Meng, and Ermon]{lou2024discrete}
Aaron Lou, Chenlin Meng, and Stefano Ermon.
\newblock Discrete diffusion modeling by estimating the ratios of the data
  distribution.
\newblock In \emph{International Conference on Machine Learning}, 2024.

\bibitem[Matthews et~al.(2022)Matthews, Arbel, Rezende, and
  Doucet]{matthews2022continual}
Alex Matthews, Michael Arbel, Danilo~Jimenez Rezende, and Arnaud Doucet.
\newblock Continual repeated annealed flow transport {Monte Carlo}.
\newblock In \emph{International Conference on Machine Learning}, pp.\
  15196--15219, 2022.

\bibitem[Midgley et~al.(2023)Midgley, Stimper, Simm, Sch{\"o}lkopf, and
  Hern{\'a}ndez-Lobato]{midgley2023flow}
Laurence~Illing Midgley, Vincent Stimper, Gregor N.~C. Simm, Bernhard
  Sch{\"o}lkopf, and Jos{\'e}~Miguel Hern{\'a}ndez-Lobato.
\newblock Flow annealed importance sampling bootstrap.
\newblock In \emph{International Conference on Learning Representations}, 2023.

\bibitem[Neal(2001)]{neal2001annealed}
Radford~M. Neal.
\newblock Annealed importance sampling.
\newblock \emph{Statistics and Computing}, 11\penalty0 (2):\penalty0 125--139,
  2001.

\bibitem[Nisonoff et~al.(2025)Nisonoff, Xiong, Allenspach, and
  Listgarten]{nisonoff2025unlocking}
Hunter Nisonoff, Junhao Xiong, Stephan Allenspach, and Jennifer Listgarten.
\newblock Unlocking guidance for discrete state-space diffusion and flow
  models.
\newblock In \emph{International Conference on Learning Representations}, 2025.

\bibitem[Ou et~al.(2025)Ou, Zhang, and Li]{ou2025dnfs}
Zijing Ou, Ruixiang Zhang, and Yingzhen Li.
\newblock Discrete neural flow samplers with locally equivariant transformer.
\newblock In \emph{Advances in Neural Information Processing Systems}, 2025.

\bibitem[Rissanen et~al.(2025)Rissanen, Ouyang, He, Chen, Heinonen, Solin, and
  Hern{\'a}ndez-Lobato]{rissanen2025progressive}
Severi Rissanen, Ruikang Ouyang, Jiajun He, Wenlin Chen, Markus Heinonen, Arno
  Solin, and Jos{\'e}~Miguel Hern{\'a}ndez-Lobato.
\newblock Progressive tempering sampler with diffusion.
\newblock In \emph{Proceedings of the 42nd International Conference on Machine
  Learning}, volume 267 of \emph{Proceedings of Machine Learning Research},
  pp.\  51724--51746, 2025.

\bibitem[Sanokowski et~al.(2024)Sanokowski, Hochreiter, and
  Lehner]{sanokowski2024diffuco}
Sebastian Sanokowski, Sepp Hochreiter, and Sebastian Lehner.
\newblock A diffusion model framework for unsupervised neural combinatorial
  optimization.
\newblock In \emph{International Conference on Machine Learning}, 2024.

\bibitem[Sanokowski et~al.(2025)Sanokowski, Berghammer, Ennenmoser, Wang,
  Hochreiter, and Lehner]{sanokowski2025sdds}
Sebastian Sanokowski, Wilhelm Berghammer, Martin Ennenmoser, Haoyu~Peter Wang,
  Sepp Hochreiter, and Sebastian Lehner.
\newblock Scalable discrete diffusion samplers: Combinatorial optimization and
  statistical physics.
\newblock In \emph{International Conference on Learning Representations}, 2025.

\bibitem[Schiff et~al.(2025)Schiff, Sahoo, Phung, Wang, Boshar, Dalla-Torre,
  de~Almeida, Rush, Pierrot, and Kuleshov]{schiff2025simple}
Yair Schiff, Subham~Sekhar Sahoo, Hao Phung, Guanghan Wang, Sam Boshar, Hugo
  Dalla-Torre, Bernardo~P. de~Almeida, Alexander~M. Rush, Thomas Pierrot, and
  Volodymyr Kuleshov.
\newblock Simple guidance mechanisms for discrete diffusion models.
\newblock In \emph{International Conference on Learning Representations}, 2025.

\bibitem[Schopmans \& Friederich(2025)Schopmans and
  Friederich]{schopmans2025temperature}
Henrik Schopmans and Pascal Friederich.
\newblock Temperature-annealed {Boltzmann} generators.
\newblock In \emph{Proceedings of the 42nd International Conference on Machine
  Learning}, volume 267 of \emph{Proceedings of Machine Learning Research},
  pp.\  53467--53500, 2025.

\bibitem[Sun et~al.(2023)Sun, Yu, Dai, Schuurmans, and Dai]{sun2023score}
Haoran Sun, Lijun Yu, Bo~Dai, Dale Schuurmans, and Hanjun Dai.
\newblock Score-based continuous-time discrete diffusion models.
\newblock In \emph{International Conference on Learning Representations}, 2023.

\bibitem[Syed et~al.(2026)Syed, Bouchard-C{\^o}t{\'e}, Chern, and
  Doucet]{syed2026optimized}
Saifuddin Syed, Alexandre Bouchard-C{\^o}t{\'e}, Kevin Chern, and Arnaud
  Doucet.
\newblock Optimized annealed sequential monte carlo samplers.
\newblock \emph{Journal of the Royal Statistical Society Series B: Statistical
  Methodology}, pp.\  qkag082, 2026.
\newblock \doi{10.1093/jrsssb/qkag082}.

\bibitem[Wan et~al.(2026)Wan, Ouyang, Xie, Fang, Zha, and
  Cheng]{wan2026discrete}
Zhengyan Wan, Yidong Ouyang, Liyan Xie, Fang Fang, Hongyuan Zha, and Guang
  Cheng.
\newblock Discrete guidance matching: Exact guidance for discrete flow
  matching.
\newblock In \emph{International Conference on Learning Representations}, 2026.

\bibitem[Wu et~al.(2019)Wu, Wang, and Zhang]{wu2019solving}
Dian Wu, Lei Wang, and Pan Zhang.
\newblock Solving statistical mechanics using variational autoregressive
  networks.
\newblock \emph{Physical Review Letters}, 122\penalty0 (8):\penalty0 080602,
  2019.

\bibitem[Zanella(2020)]{zanella2020informed}
Giacomo Zanella.
\newblock Informed proposals for local {MCMC} in discrete spaces.
\newblock \emph{Journal of the American Statistical Association}, 115\penalty0
  (530):\penalty0 852--865, 2020.
\newblock \doi{10.1080/01621459.2019.1585255}.

\bibitem[Zhang et~al.(2022)Zhang, Liu, and Liu]{zhang2022langevin}
Ruqi Zhang, Xingchao Liu, and Qiang Liu.
\newblock A {L}angevin-like sampler for discrete distributions.
\newblock In \emph{Proceedings of the 39th International Conference on Machine
  Learning}, volume 162 of \emph{Proceedings of Machine Learning Research},
  pp.\  26375--26396, 2022.

\bibitem[Zhou et~al.(2016)Zhou, Johansen, and Aston]{zhou2016automatic}
Yan Zhou, Adam~M. Johansen, and John A.~D. Aston.
\newblock Toward automatic model comparison: An adaptive sequential monte carlo
  approach.
\newblock \emph{Journal of Computational and Graphical Statistics}, 25\penalty0
  (3):\penalty0 701--726, 2016.

\bibitem[Zhu et~al.(2025)Zhu, Guo, Choi, Liu, Chen, and Tao]{zhu2025mdns}
Yuchen Zhu, Wei Guo, Jaemoo Choi, Guan-Horng Liu, Yongxin Chen, and Molei Tao.
\newblock {MDNS}: Masked diffusion neural sampler via stochastic optimal
  control.
\newblock In \emph{Advances in Neural Information Processing Systems}, 2025.

\end{thebibliography}
\bibliographystyle{iclr2027_template/iclr2027_conference}

\clearpage
\appendix
\section{Algorithms}
\label{app:algorithms}
\raggedbottom

Algorithms~\ref{alg:iedg-training} and~\ref{alg:direct-q-sampling} turn the
construction in Section~\ref{sec:method} into the stagewise training and sampling
procedures used throughout the paper.

\begin{algorithm}[H]
    \caption{Iterative Exact Discrete Guidance training}
    \label{alg:iedg-training}
    \small
    \begin{algorithmic}[1]
        \Require Energy $E$; base posterior
        $p_{\mathrm{ref},1\mid t}$; conditional path $p_{t\mid1}$;
        time law $\tau$; rESS thresholds $(\eta_{\mathrm{med}},\eta_{10})$;
        stage budget $K_{\max}$; buffer size $n$; optional frozen path
        $\mathcal A^\star$
        \Ensure Final guide $G_{\psi_K}$ and endpoints
        $\mathcal A=(\alpha_0,\ldots,\alpha_K)$

        \State Set $G_{\psi_0}\equiv1$, $\alpha_0\gets0$, $k\gets1$, and draw
        $\mathcal B_1=\{X^{(i)}\}_{i=1}^{n}
        \overset{\mathrm{i.i.d.}}{\sim}\refdist$

        \While{$\alpha_{k-1}<1$}
            \If{$\mathcal A^\star$ is provided}
                \State Set $\alpha_k\gets\alpha_k^\star$
            \Else
                \State Set $K_k^{\mathrm{rem}}\gets K_{\max}-k+1$; compute
                $\Delta_k^{\mathrm{ESS}}$ by
                \eqref{eq:groupwise-ess-path-rule}, and set
                $\Delta\alpha_k$ by \eqref{eq:ess-path-reachability}
                \State Set $\alpha_k\gets\alpha_{k-1}+\Delta\alpha_k$
            \EndIf
            \State Set
            $w_k(x)\gets e^{-(\alpha_k-\alpha_{k-1})E(x)}$ and freeze
            $\widehat c_k\gets|\mathcal B_k|^{-1}
            \sum_{x\in\mathcal B_k}w_k(x)$
            \State Initialize $\psi_k\gets\psi_{k-1}$

            \For{each stochastic training update}
                \State Sample
                $X_1\sim\operatorname{Unif}(\mathcal B_k)$,
                $t\sim\tau$, and
                $X_t\sim p_{t\mid1}(\cdot\mid X_1)$
                \State Evaluate the frozen teacher
                $\widehat p_{k-1,1\mid t}^{d}\gets
                p_{\mathrm{ref},1\mid t}^{d}$ if $k=1$, and
                $\widehat p_{k-1,1\mid t}^{d}\gets
                q_{\psi_{k-1},1\mid t}^{d}$ otherwise
                \State Form, for each $d$,
                \[
                R_{k,t}^{d}
                \gets
                \frac{w_k(X_1)}{\widehat c_k}
                \frac{
                  \widehat p_{k-1,1\mid t}^{d}(X_1^d\mid X_t)}
                {
                  p_{\mathrm{ref},1\mid t}^{d}(X_1^d\mid X_t)}
                \]
                \State Update $\psi_k$ using the bridge objective
                $\mathcal L_k$ in \eqref{eq:bridge-objective}
            \EndFor

            \State Freeze the stage guide $G_{\psi_k}$
            \If{$\alpha_k<1$}
                \State Generate the $n$ states of $\mathcal B_{k+1}$ by repeated
                calls to Algorithm~\ref{alg:direct-q-sampling} using $G_{\psi_k}$
            \EndIf
            \State $k\gets k+1$
        \EndWhile

        \State Set $K\gets k-1$ and assert $\alpha_K=1$
        \State \Return $G_{\psi_K}$ and $\mathcal A$
    \end{algorithmic}
\end{algorithm}

\begin{algorithm}[H]
    \caption{Direct-$q$ posterior-marginal tau-leaping}
    \label{alg:direct-q-sampling}
    \small
    \begin{algorithmic}[1]
        \Require Positive guide $G_\psi$; analytic base posterior
        $p_{\mathrm{ref},1\mid t}$; interpolation schedule $\kappa$; number of
        steps $N$
        \Ensure Terminal state $X_N$
        \State Sample $X_0\sim\nu^{\otimes D}$ and set $\Delta t\gets1/N$
        \For{$i=0,\ldots,N-1$}
            \State $t_i\gets i/N$
            \State Compute, for every $d$ and $z\in\alphabet$,
            \Statex \hspace{1.6em}$\displaystyle
            q_{\psi,1\mid t_i}^{d}(z\mid X_i)
            \gets
            \frac{p_{\mathrm{ref},1\mid t_i}^{d}(z\mid X_i)
            G_{\psi,t_i}^{d}(z,X_i)}
            {\sum_{a\in\alphabet}p_{\mathrm{ref},1\mid t_i}^{d}(a\mid X_i)
            G_{\psi,t_i}^{d}(a,X_i)}$.
            \ForAll{$d=1,\ldots,D$ \textbf{ in parallel}}
                \State $\displaystyle
                \lambda_i^d(z)\gets
                \frac{\dot\kappa_{t_i}}{1-\kappa_{t_i}}
                q_{\psi,1\mid t_i}^{d}(z\mid X_i)
                \indc{z\neq X_i^d}$ and
                $\Lambda_i^d\gets\sum_{z\in\alphabet}\lambda_i^d(z)$
                \State Sample
                $J_i^d\sim\operatorname{Bernoulli}
                (1-e^{-\Delta t\Lambda_i^d})$
                \If{$J_i^d=1$}
                    \State Sample $\widetilde X_{i+1}^d$ from
                    $\Pr(\widetilde X_{i+1}^d=z)=
                    \lambda_i^d(z)/\Lambda_i^d$
                \Else
                    \State $\widetilde X_{i+1}^d\gets X_i^d$
                \EndIf
            \EndFor
            \State Commit the parallel update $X_{i+1}\gets\widetilde X_{i+1}$
        \EndFor
        \State \Return $X_N$
    \end{algorithmic}
\end{algorithm}

Algorithm~\ref{alg:iedg-training} selects each endpoint from the current source buffer
without $\mathcal A^\star$. For reported experiments, one
development run executes this adaptive branch; its endpoints are then frozen
and replayed across formal runs to reduce path-induced variation.
Algorithm~\ref{alg:direct-q-sampling} holds the state-dependent rates fixed on
each interval, integrates the resulting exit rate through the exponential jump
probability, and permits at most one jump per coordinate. It performs $N$
posterior evaluations and no additional terminal denoising step. The resulting
finite-step error is isolated in Appendix~\ref{app:practical-residuals}.

\subsection{ESS path construction and experimental conventions}

\paragraph{Adaptive ESS path construction.}
\label{app:ess-calibration}
The increment actually applied at stage $k$ is
\begin{equation}
    \Delta\alpha_k
    :=\min\left\{1-\alpha_{k-1},
      \max\left\{\Delta_k^{\mathrm{ESS}},
      \frac{1-\alpha_{k-1}}{K_k^{\mathrm{rem}}}\right\}
    \right\}.
    \label{eq:ess-path-reachability}
\end{equation}
Here $K_k^{\mathrm{rem}}$ is the number of available transitions including
stage $k$, and $\Delta_k^{\mathrm{ESS}}$ is the largest candidate satisfying
the groupwise overlap constraints
\begin{equation}
\begin{aligned}
    \Delta_k^{\mathrm{ESS}}
    :=\max\bigl\{0\leq\Delta\leq1-\alpha_{k-1}:\;&
      \operatorname{Median}_{g}
      \bigl[\widehat{\ress}_{k,g}(\Delta)\bigr]
        \geq\eta_{\mathrm{med}},\\[-2pt]
      &Q_{0.1}^{(g)}
      \bigl[\widehat{\ress}_{k,g}(\Delta)\bigr]\geq\eta_{10}
    \bigr\}.
\end{aligned}
    \label{eq:groupwise-ess-path-rule}
\end{equation}
In this expression, $Q_{0.1}^{(g)}[\cdot]$ is the empirical lower decile over
groups and
\begin{equation}
    \widehat{\ress}_{k,g}(\Delta)
    :=\frac{
      \left(\sum_{x\in\mathcal B_{k,g}}e^{-\Delta E(x)}\right)^2
    }{
      |\mathcal B_{k,g}|
      \sum_{x\in\mathcal B_{k,g}}e^{-2\Delta E(x)}
    },
    \label{eq:groupwise-empirical-ress}
\end{equation}
is the empirical rESS of group $\mathcal B_{k,g}$. The analytic first stage
draws path-selection states from $\refdist$; later stages use the fixed source
buffer. Lattice buffers are split into contiguous groups, whereas each graph
and its source states form one group. We evaluate
\eqref{eq:groupwise-empirical-ress} by log-sum-exp and find the largest
admissible candidate by bisection. The reachability term in
\eqref{eq:ess-path-reachability} guarantees arrival at $\alpha=1$ within the
stage budget; the groupwise statistics are finite-buffer estimates of the
population rESS in \eqref{eq:population-ess-ratio}. Frozen endpoints and
numerical settings are reported in Appendix~\ref{app:frozen-ess-path}.

\paragraph{Stagewise response normalization.}
\label{app:fixed-stage-normalizer}
The source buffer fixes the positive stage scale
\begin{equation*}
    \widehat c_k
    :=\frac{1}{|\mathcal B_k|}
      \sum_{x\in\mathcal B_k}\exp\{-\Delta\alpha_kE(x)\}>0,
\end{equation*}
which is held constant throughout training; graph-conditioned targets use one
$\widehat c_k(G)$ per graph buffer. Because this factor is independent of
$(t,x_t,d,z)$ within an instance, it rescales the population guide uniformly and
cancels in the normalized posterior \eqref{eq:fixed-base-parameterization}.

\paragraph{Experimental source-promotion exponent.}
For the two lattice transports specified in
Appendix~\ref{app:source-promotion}, an intermediate source buffer may be
generated from
\begin{equation}
    q_{\psi_k,1\mid t}^{d,(\gamma)}(z\mid x_t)
    =\frac{
      p_{\mathrm{ref},1\mid t}^{d}(z\mid x_t)
      [G_{\psi_k,t}^{d}(z,x_t)]^\gamma
    }{
      \sum_{a\in\alphabet}p_{\mathrm{ref},1\mid t}^{d}(a\mid x_t)
      [G_{\psi_k,t}^{d}(a,x_t)]^\gamma
    },
    \label{eq:source-promotion-posterior}
\end{equation}
where $\gamma=1$ recovers the canonical posterior. Validation selects
$\gamma_{\mathrm{src}}\in\{1,1.1\}$ together with the source checkpoint; this
pair is used only to construct the next buffer and its frozen posterior.
All exactness statements and terminal samples use
$\gamma_{\mathrm{term}}=1$. Thus nonunit source promotion is a finite-model
bias-compensation heuristic accounted for by
$r_{k,T}^{\mathrm{pr}}$ in Appendix~\ref{app:practical-residuals}.
\flushbottom

\section{Additional Related Works}
\label{app:related-work}

\subsection{Discrete diffusion, flow, and guidance}

Categorical diffusion models reverse multinomial or structured corruption
kernels \citep{hoogeboom2021argmax,austin2021structured}. Continuous-time
formulations represent denoising by a CTMC
\citep{campbell2022continuous} and estimate reverse rates or density ratios
\citep{sun2023score,lou2024discrete}. Discrete Flow Matching specifies a
conditional probability path and marginalizes endpoint-conditional rates using
a learned posterior; this formulation includes discrete diffusion paths as
special cases \citep{campbell2024generative,gat2024discrete}.

Guidance modifies or reweights reverse dynamics using conditional signals, which
may come from an auxiliary model or a jointly trained conditional
parameterization. \citet{schiff2025simple} derive classifier-based and classifier-free
reweighting rules for discrete diffusion, whereas
\citet{nisonoff2025unlocking} formulate guidance through CTMC transition rates,
including an exact but generally costly correction and practical
approximations. Recent work also characterizes the distributions and dynamics
induced by classifier-free guidance in masked diffusion
\citep{he2026guidance}. DGM takes a different route: it proves that a terminal
density ratio yields an exact endpoint-posterior correction through a
conditional expectation \citep{wan2026discrete}. IEDG uses this identity
unchanged; its contribution is a multistage construction that defines each
local ratio against the available sampler and accumulates all corrections on a
shared posterior base.

\subsection{Neural samplers for unnormalized discrete targets}

Variational autoregressive networks minimize variational free energy using
tractable likelihoods \citep{wu2019solving}, while GFlowNets train constructive
policies whose terminal probabilities are proportional to a reward
\citep{bengio2023gflownet}. DiffUCO minimizes a reverse-KL upper bound for
data-free combinatorial optimization \citep{sanokowski2024diffuco}; SDDS avoids
backpropagation through a complete diffusion trajectory using policy-gradient
and self-normalized neural importance-sampling objectives
\citep{sanokowski2025sdds}.

An alternative line improves inference-time Markov transitions rather than
amortizing a complete sampler. Locally balanced proposals use pointwise target
information to construct efficient discrete-space Metropolis--Hastings moves
\citep{zanella2020informed}. Gradient-informed Gibbs proposals and discrete
Langevin proposals extend this principle to high-dimensional energy models
\citep{grathwohl2021oops,zhang2022langevin}. Their Metropolis-corrected forms
retain the target as the stationary law, but still require a chain of local
transitions for each new sample; neural samplers instead move most of this cost
to training.

LEAPS learns locally equivariant CTMC rates by controlling path-space importance
weights \citep{holderrieth2025leaps}. DNFS fits neural rates through Monte Carlo
Kolmogorov residuals, with control variates and coordinate descent for variance
reduction \citep{ou2025dnfs}. MDNS derives masked-diffusion objectives from
stochastic optimal control \citep{zhu2025mdns}, and MetaDNS incorporates
well-tempered metadynamics into diffusion or autoregressive samplers to improve
metastable-mode exploration \citep{du2026metadns}. DASBS extends adjoint
matching to finite-state CTMC Schr\"odinger bridges using a cyclic-group
construction \citep{guo2026dasbs}. Concurrent work develops off-policy training
for discrete diffusion samplers and data-to-energy bridges
\citep{carter2026offpolicy}. These objectives differ from IEDG's
conditional energy-ratio regression and same-base posterior accumulation.

\subsection{Progressive tempering and learned transport}

Annealed importance sampling and sequential Monte Carlo connect a tractable
reference to a target through intermediate distributions, importance weighting,
and Markov transitions or resampling
\citep{neal2001annealed,delmoral2006sequential}. Adaptive SMC uses empirical
particle statistics, notably ESS, to select intermediate temperatures and tune
subsequent transitions
\citep{jasra2011inference,beskos2016convergence,zhou2016automatic}; this
scheduling principle predates IEDG. Thermodynamic length provides a metric of
separation along equilibrium paths and is linked to Fisher information
\citep{crooks2007measuring}. IEDG uses rESS to control consecutive overlap, and
its analysis additionally makes the induced change-of-measure factor explicit
in the learning-to-sampling bound alongside conditional coverage.

The closest continuous-space analogues combine progressive targets with
self-generated training data. Temperature-Annealed Boltzmann Generators
(TA-BG) first fit a normalizing flow at high temperature and then retrain it at
successively lower temperatures using importance-weighted samples from the
current generator; reverse ESS measures the quality of this reweighting
\citep{schopmans2025temperature}. FAB instead bootstraps a flow from
annealed-importance-weighted samples \citep{midgley2023flow}. iDEM alternates
reverse-diffusion sampling with denoising energy-matching updates on regions
visited by the current model, without an explicit temperature path
\citep{akhoundsadegh2024iterated}. PTSD trains diffusion samplers sequentially
across temperatures, combines higher-temperature models to initialize the next
target, and applies limited MCMC refinement before retraining
\citep{rissanen2025progressive}.

These methods share IEDG's use of easier intermediate problems and samples from
the current sampler, but their correction mechanisms differ: importance-weighted flow fitting for TA-BG and FAB, iterative reverse-SDE self-training for
iDEM, and MCMC-refined temperature transfer for PTSD. IEDG instead learns
finite-state endpoint density-ratio corrections for local Boltzmann increments
and accumulates them on a common analytic posterior base. Annealed Flow
Transport and CRAFT likewise fit normalizing-flow maps between successive
annealed targets
\citep{arbel2021annealed,matthews2022continual}. Most closely, PDNS performs
proximal updates in path-measure space and solves each subproblem with a weighted
denoising objective \citep{guo2026pdns}. IEDG likewise uses a geometric
Boltzmann path, which is not itself novel, but transports by posterior-exact
local tilts and absorbs every correction into a common analytic base. Unlike
annealed map-composition methods, its intermediate stages construct the
training-time posterior; terminal inference uses one accumulated guide.

\section{Extended Background: Marginalizing Conditional Rates}
\label{app:background}

This appendix records the finite-state marginalization argument used by the
posterior-marginal sampler. It also clarifies which part of reverse sampling is an
exact continuous-time identity and which part is introduced by numerical
discretization.

\begin{proposition}[Marginalization of endpoint-conditional rates]
\label{prop:rate-marginalization}
Let $X_1\sim q_1$ and suppose that, conditional on $X_1=x_1$, a finite-state
time-inhomogeneous Markov chain has marginal $q_{t\mid 1}(\cdot\mid x_1)$ and
off-diagonal rate $u_t^q(z,x\mid x_1)$ from $x$ to $z$. Let
\begin{equation}
    u_t^q(z,x)
    :=\mathbb{E}[u_t^q(z,x\mid X_1)\mid X_t=x].
    \label{eq:marginalized-rate}
\end{equation}
Then the unconditional marginal
$q_t(x)=\sum_{x_1}q_1(x_1)q_{t\mid 1}(x\mid x_1)$ satisfies the forward equation
with off-diagonal rate $u_t^q$.
\end{proposition}

\begin{proof}
For each endpoint $x_1$, the conditional marginal satisfies
\begin{equation*}
    \partial_t q_{t\mid 1}(x\mid x_1)
    =\sum_{z\neq x}
      \left[
        q_{t\mid 1}(z\mid x_1)u_t^q(x,z\mid x_1)
        -q_{t\mid 1}(x\mid x_1)u_t^q(z,x\mid x_1)
      \right].
\end{equation*}
Multiplying by $q_1(x_1)$, summing over $x_1$, and using Bayes' rule gives
\begin{align}
    \partial_t q_t(x)
    &=\sum_{z\neq x}
      \left[
        q_t(z)\mathbb{E}[u_t^q(x,z\mid X_1)\mid X_t=z]
        -q_t(x)\mathbb{E}[u_t^q(z,x\mid X_1)\mid X_t=x]
      \right] \notag\\
    &=\sum_{z\neq x}
      \left[
        q_t(z)u_t^q(x,z)
        -q_t(x)u_t^q(z,x)
      \right]. \notag
\end{align}
This is the desired forward equation; all interchanges are finite sums.
\end{proof}

For the coordinate-wise uniform-replacement path, the off-diagonal conditional
rate for changing coordinate $d$ from $x^d$ to $z$ is
\begin{equation*}
    u_t^{q,d}(z,x^d\mid x_1^d)
    =\frac{\dot\kappa_t}{1-\kappa_t}
      \indc{z=x_1^d}\indc{z\neq x^d}.
\end{equation*}
Conditioning on $X_t=x$ in \eqref{eq:marginalized-rate} immediately yields
\begin{equation*}
    u_t^{q,d}(z,x)
    =\frac{\dot\kappa_t}{1-\kappa_t}
      q_{1\mid t}^{d}(z\mid x)\indc{z\neq x^d}.
\end{equation*}
Therefore, an exact coordinate posterior gives exact marginal rates for the
intended path. Euler or tau-leap
simulation of these rates is a separate numerical approximation.

\section{Complete Proofs}
\label{app:proofs}

Sections~D.1--D.3 establish the population-exact stage bridge, and Section~D.4
develops the ESS geometry and its $L^2$ change-of-measure control.
Sections~D.5--D.6 connect Bregman excess risk to finite-time sampling error and
define the resulting stagewise remainder. Section~D.7 propagates this error
across stages and proves ideal recovery.

\subsection{Bregman conditional mean and excess identity}

We record the standard conditional-mean property of the DGM Bregman loss
\citep{banerjee2005optimality,wan2026discrete} and the associated excess-risk
identity used in Section~\ref{app:one-stage-calibration}.

\begin{lemma}[Positive conditional Bregman predictor]
\label{lem:conditional-bregman}
Let $(R,Y)$ be random variables with $R>0$ integrable. Among positive
measurable predictors of finite risk,
the population DGM loss is uniquely minimized, up to null sets, by
\begin{equation*}
    g^\star(Y)=\mathbb E[R\mid Y].
\end{equation*}
\end{lemma}

\begin{proof}
Set $a(Y):=\mathbb E[R\mid Y]>0$. By the tower property, the
predictor-dependent risk is $\mathbb E[g(Y)-a(Y)\log g(Y)]$. For fixed $a>0$,
$g-a\log g$ has derivative $1-a/g$ and second derivative $a/g^2>0$; hence its
unique minimizer is $g=a$. Moreover, with $\phi(u):=u-1-\log u$,
\begin{equation}
    \bigl[g-a\log g\bigr]-\bigl[a-a\log a\bigr]
    =a\phi\!\left(\frac{g}{a}\right).
    \label{eq:fkl-excess-risk}
\end{equation}
The right-hand side is nonnegative and vanishes only at $g=a$.
\end{proof}

\subsection{DGM posterior reweighting identity}

Let $p_1$ be the source terminal law, let $w>0$ satisfy
$0<\mathbb E_{p_1}[w(X_1)]<\infty$, and let
$q_1=\tiltop_w(p_1)$, i.e.,
$q_1(x_1)=p_1(x_1)w(x_1)/Z_w$ with $Z_w:=\mathbb E_{p_1}[w(X_1)]$.
The source and target use the same conditional path $p_{t\mid1}$. Writing
$p_t$ and $q_t$ for the corresponding noised marginals, all conditional
expectations below are under the source joint law
$p_1(x_1)p_{t\mid1}(x_t\mid x_1)$. Bayes' rule gives
\begin{equation*}
    \frac{q_t(x_t)}{p_t(x_t)}
    =\frac{\mathbb E[w(X_1)\mid X_t=x_t]}{Z_w}.
\end{equation*}
Hence, for any positive-probability context,
\begin{align}
    q_{1\mid t}(x_1\mid x_t)
    &=\frac{q_1(x_1)p_{t\mid1}(x_t\mid x_1)}{q_t(x_t)} \notag\\
    &=\frac{p_1(x_1)p_{t\mid1}(x_t\mid x_1)}{p_t(x_t)}
      \frac{w(x_1)p_t(x_t)}{Z_wq_t(x_t)} \notag\\
    &=p_{1\mid t}(x_1\mid x_t)
      \frac{w(x_1)}{\mathbb E[w(X_1)\mid X_t=x_t]}.
    \label{eq:full-posterior-reweighting}
\end{align}
Thus the unknown terminal normalizer $Z_w$ cancels inside the posterior.
Writing $x_1^{-d}$ for all endpoint coordinates except $d$, marginalization
over $x_1^{-d}$ gives
\begin{align}
    q_{1\mid t}^{d}(z\mid x_t)
    &=\sum_{x_1^{-d}}q_{1\mid t}(z,x_1^{-d}\mid x_t) \notag\\
    &=\frac{p_{1\mid t}^{d}(z\mid x_t)
      \mathbb E[w(X_1)\mid X_1^d=z,X_t=x_t]}
      {\mathbb E[w(X_1)\mid X_t=x_t]}.
    \notag
\end{align}
Finally, the law of total expectation gives
\begin{equation*}
    \mathbb E[w(X_1)\mid X_t=x_t]
    =\sum_{a\in\alphabet}p_{1\mid t}^{d}(a\mid x_t)
      \mathbb E[w(X_1)\mid X_1^d=a,X_t=x_t].
\end{equation*}
Substitution proves
\eqref{eq:dgm-identity}. This is the DGM posterior-guidance identity, restated
to expose the normalizer cancellation used at every IEDG stage.

\subsection{Posterior-exact bridge and same-base accumulation}

For Proposition~\ref{thm:stage-exactness}, assume that, for $\tau$-almost every
$t$ and $p_{k-1,t}$-almost every $x_t$, the source, fixed-base, and frozen
coordinate posteriors are strictly positive for every $d$ and
$z\in\alphabet$, and that the Bregman risks are finite. These conditions hold
at $t<1$ for the uniform-replacement path with full-support reference and
source.

\begin{proof}[Proof of Proposition~\ref{thm:stage-exactness}]
Fix a coordinate $d$ and condition on a training context
$(X_1^d,X_t,t)=(z,x_t,t)$. By Lemma~\ref{lem:conditional-bregman}, the unique
conditional-risk minimizer is the conditional mean of the sampled response:
\begin{align}
    G_{k,t}^{\star,d}(z,x_t)
    &=\mathbb{E}[R_{k,t}^{d}\mid X_1^d=z,X_t=x_t,t] \notag\\
    &=\frac{1}{\widehat c_k}A_{k-1,t}^{d}(z,x_t)
      \mathbb{E}[w_k(X_1)\mid X_1^d=z,X_t=x_t,t] \notag\\
    &=\frac{1}{\widehat c_k}A_{k-1,t}^{d}(z,x_t)h_{k,t}^{d}(z,x_t).
    \label{eq:bridge-proof-optimum}
\end{align}
Here the second equality uses that $A_{k-1,t}^{d}(X_1^d,X_t)$ is fixed by the
regression context. Under the posterior equality assumed in
Proposition~\ref{thm:stage-exactness},
\begin{equation*}
    p_{\mathrm{ref},1\mid t}^{d}(z\mid x_t)
    A_{k-1,t}^{d}(z,x_t)
    =p_{k-1,1\mid t}^{d}(z\mid x_t).
\end{equation*}
Multiplying \eqref{eq:bridge-proof-optimum} by the fixed-base posterior and
normalizing over categories therefore yields
\begin{align}
    \frac{
      p_{\mathrm{ref},1\mid t}^{d}(z\mid x_t)G_{k,t}^{\star,d}(z,x_t)
    }{
      \sum_{a\in\alphabet}p_{\mathrm{ref},1\mid t}^{d}(a\mid x_t)
      G_{k,t}^{\star,d}(a,x_t)
    }
    &=
    \frac{
      p_{k-1,1\mid t}^{d}(z\mid x_t)h_{k,t}^{d}(z,x_t)
    }{
      \sum_{a\in\alphabet}p_{k-1,1\mid t}^{d}(a\mid x_t)
      h_{k,t}^{d}(a,x_t)
    }
    =q_{k,1\mid t}^{d}(z\mid x_t),
    \label{eq:same-base-recovers-teacher}
\end{align}
where the last equality is the coordinate DGM identity
\eqref{eq:dgm-identity}. Lemma~\ref{lem:conditional-bregman} gives uniqueness
for $\tau(t)p_{k-1,t}(x_t)$-almost every training context. Replacing
$\widehat c_k$ by any $\widehat c_k'>0$ multiplies the population guide by the
common factor $\widehat c_k/\widehat c_k'$,
which cancels in \eqref{eq:same-base-recovers-teacher}.
\end{proof}

\subsection{ESS-controlled annealing geometry}
\label{app:ess-geometry-proof}

\begin{proof}[Proof of Proposition~\ref{prop:ess-geometry}]
Abbreviate $\mu=\mu_{k-1}$ and
$Z_\mu(\Delta)=\mathbb E_\mu[w_\Delta(X)]$. Since $w_\Delta>0$,
$\tiltop_{w_\Delta}(\mu)$ and $\mu$ have the same support and
\begin{equation*}
    \frac{\tiltop_{w_\Delta}(\mu)(x)}{\mu(x)}
    =\frac{w_\Delta(x)}{Z_\mu(\Delta)}.
\end{equation*}
By the definition
$D_2(P\Vert Q)=\log\sum_{x:Q(x)>0} P(x)^2/Q(x)$,
\begin{align}
    D_2(\tiltop_{w_\Delta}(\mu)\Vert\mu)
    &=\log\sum_x\mu(x)
      \left(\frac{w_\Delta(x)}{Z_\mu(\Delta)}\right)^2 \notag\\
    &=\log\frac{\mathbb E_\mu[w_\Delta(X)^2]}
      {\mathbb E_\mu[w_\Delta(X)]^2}
     =-\log \ress_k(\Delta).
     \notag
\end{align}
This proves \eqref{eq:ess-renyi-identity}. Convexity of
$\Psi_\mu(s):=\log\mathbb E_\mu[e^{-sE(X)}]$ also implies that
$-\log\ress_k(\Delta)=\Psi_\mu(2\Delta)-2\Psi_\mu(\Delta)$ is non-decreasing;
hence the rESS search in \eqref{eq:ess-path-rule} is well posed.
Equivalently, for
$\bar w_\Delta=w_\Delta/\mathbb E_\mu[w_\Delta]$,
$\operatorname{Var}_\mu[\bar w_\Delta]=\ress_k(\Delta)^{-1}-1$.

For the ideal path, let
$Z(\alpha):=\sum_x\refdist(x)e^{-\alpha E(x)}$ and write
$\Phi(\alpha)=\log Z(\alpha)$. Direct substitution gives
\begin{align}
    -\log \ress_\alpha(\Delta)
    &=\Phi(\alpha)+\Phi(\alpha+2\Delta)-2\Phi(\alpha+\Delta).
    \notag
\end{align}
Finiteness of $\state$ makes $\Phi$ analytic, and differentiating the log-partition
function gives $\Phi''(\alpha)=\operatorname{Var}_{\pi_\alpha}[E]$. Taylor expansion
around $\alpha$ therefore proves \eqref{eq:ess-local-expansion}.

Since
$\partial_\alpha\log\pi_\alpha(x)=-E(x)+\mathbb E_{\pi_\alpha}[E]$, the
Fisher information of the ideal family is
$\mathcal I(\alpha)=\operatorname{Var}_{\pi_\alpha}(E)$ and its thermodynamic line
element is $\dd\ell_{\mathrm{th}}=\sqrt{\mathcal I(\alpha)}\dd\alpha$. Maintaining
constant relative ESS therefore matches squared thermodynamic line elements through
second order in $\Delta$.
\end{proof}

\paragraph{Noised occupation.}
For the applied increment, abbreviate $\mu=\mu_{k-1}$,
$w=w_k$, and $\bar w=w/\mathbb E_\mu[w]$. Let $p_{k-1,t}$ and $q_{k,t}$
be the noised marginals of $\mu$ and $\tiltop_w(\mu)$ under their shared
conditional path. Bayes' rule gives
\begin{equation*}
    \omega_{k,t}(x_t)
    :=\frac{q_{k,t}(x_t)}{p_{k-1,t}(x_t)}
    =\mathbb E[\bar w(X_1)\mid X_t=x_t].
\end{equation*}
Conditional Jensen and \eqref{eq:ess-renyi-identity} then imply, for every
$t<1$,
\begin{equation}
    \mathbb E_{X_t\sim p_{k-1,t}}[\omega_{k,t}(X_t)^2]
    \leq\mathbb E_{X_1\sim\mu}[\bar w(X_1)^2]
    =\ress_k(\Delta\alpha_k)^{-1}.
    \label{eq:noised-occupation-second-moment}
\end{equation}
Thus the forward kernel contracts the order-$2$ R\'enyi change of measure; this
is the $L^2$ control used in Section~\ref{app:one-stage-calibration}.

\subsection{Positive-tilt perturbation}

The following elementary stability bound supports the multistage recursion and
the source-posterior-mismatch accounting; it is not an additional exactness claim.

\begin{lemma}[Stability of a positive tilt]
\label{lem:positive-tilt-stability}
Let $0<m\leq w(x)\leq M<\infty$. For any probability measures $P$ and $Q$
on $\state$, write $P(f):=\mathbb E_P[f(X)]$ and similarly for $Q(f)$, and let
$a\vee b:=\max\{a,b\}$. Then
\begin{align}
    \tv\!\left(\tiltop_w(P),\tiltop_w(Q)\right)
    &\leq\frac{M}{P(w)\vee Q(w)}\tv(P,Q)
    \notag
    \\
    &\leq\frac{M}{m}\tv(P,Q).
    \notag
\end{align}
\end{lemma}

\begin{proof}
For $0\leq f\leq1$, set $\bar f=\tiltop_w(Q)(f)$. Since
$Q(w(f-\bar f))=0$ and $\osc(w(f-\bar f))\leq M$, TV duality gives
\begin{equation*}
    |\tiltop_w(P)(f)-\tiltop_w(Q)(f)|
    =\frac{|(P-Q)(w(f-\bar f))|}{P(w)}
    \leq\frac{M}{P(w)}\tv(P,Q).
\end{equation*}
Taking the supremum, interchanging $P,Q$, and using
$P(w)\vee Q(w)\geq m$ proves both inequalities.
\end{proof}

At stage $k$, source states follow $\mu_{k-1}$. Let
$p_{k-1,1\mid t}^{d}$ be their exact coordinate posterior and
$\widehat p_{k-1,1\mid t}^{d}$ the frozen posterior used by the bridge. Repeating
the conditional-mean calculation with this general frozen teacher shows that
$G_k^\star$ in \eqref{eq:bridge-population-optimum} remains the population
minimizer induced by the frozen posterior. It gives
\begin{equation}
    \bar q_{k,1\mid t}^{d}(z\mid x_t)
    =\frac{
      \widehat p_{k-1,1\mid t}^{d}(z\mid x_t)h_{k,t}^{d}(z,x_t)
    }{
      \sum_{a\in\alphabet}
      \widehat p_{k-1,1\mid t}^{d}(a\mid x_t)h_{k,t}^{d}(a,x_t)
    }.
    \label{eq:operational-stage-posterior}
\end{equation}

\begin{proposition}[Frozen-teacher posterior stability]
\label{prop:source-posterior-stability}
For a fixed positive-probability context $(d,x_t,t)$, define
\begin{equation*}
    \mathsf S_{k,t}^{d}(x_t)
    :=\frac{\max_{z\in\alphabet}h_{k,t}^{d}(z,x_t)}{
      \left[\sum_{z\in\alphabet}
        p_{k-1,1\mid t}^{d}(z\mid x_t)h_{k,t}^{d}(z,x_t)\right]
      \vee
      \left[\sum_{z\in\alphabet}
        \widehat p_{k-1,1\mid t}^{d}(z\mid x_t)
        h_{k,t}^{d}(z,x_t)\right]}.
\end{equation*}
Then
\begin{align}
    \tv\!\left(
      \bar q_{k,1\mid t}^{d}(\cdot\mid x_t),
      q_{k,1\mid t}^{d}(\cdot\mid x_t)
    \right)
    &\leq\mathsf S_{k,t}^{d}(x_t)
    \tv\!\left(
      \widehat p_{k-1,1\mid t}^{d}(\cdot\mid x_t),
      p_{k-1,1\mid t}^{d}(\cdot\mid x_t)
    \right).
    \notag
\end{align}
\end{proposition}

\begin{proof}[Proof of Proposition~\ref{prop:source-posterior-stability}]
Fix $(d,x_t,t)$ and abbreviate
$p=p_{k-1,1\mid t}^{d}(\cdot\mid x_t)$,
$\widehat p=\widehat p_{k-1,1\mid t}^{d}(\cdot\mid x_t)$, and
$h=h_{k,t}^{d}(\cdot,x_t)$. Equation~\plaineqref{eq:operational-stage-posterior}
becomes $\bar q(z)=\widehat p(z)h(z)/\sum_a\widehat p(a)h(a)$, whereas
\eqref{eq:dgm-identity} gives
\begin{equation*}
    q_{k,1\mid t}^{d}(z\mid x_t)
    =\frac{p(z)h(z)}{\sum_{a}p(a)h(a)}.
\end{equation*}
Lemma~\ref{lem:positive-tilt-stability}, applied on the categorical alphabet with
weight $h$, gives
\begin{equation*}
    \tv(\bar q_{k,1\mid t}^{d},q_{k,1\mid t}^{d})
    \leq\frac{\max_a h(a)}{p(h)\vee\widehat p(h)}\tv(\widehat p,p)
    =\mathsf S_{k,t}^{d}(x_t)\tv(\widehat p,p).
\end{equation*}
\end{proof}

\subsection{One-stage Bregman-to-sampling bound}
\label{app:one-stage-calibration}

All risks below are population risks under $\mu_{k-1}$; the finite source
buffer determines the fitted guide but is not separately decomposed.
Fix $T<1$ and write $\tau(t)$ for the training-time density. We assume
$\tau(t)>0$ almost everywhere on $[0,T]$; it may also place mass at later
times. Let $p_{k-1,t}$ and $q_{k,t}$ denote the noised marginals of
$\mu_{k-1}$ and $\widetilde\pi_k$, respectively. We first control fitting to
the posterior induced by the frozen posterior; source-posterior mismatch is added before passing from
posterior error to generator error.

Throughout this subsection, $q_{k,t}$ is the exact noised stage marginal,
$\bar q_{k,1\mid t}^{d}$ the posterior induced by the frozen posterior,
$q_{\psi_k,1\mid t}^{d}$ the learned posterior, and $Q_{k,t}$ and
$\widehat Q_{k,t}$ the exact and learned reverse generators, respectively.

For a training context $y=(d,x_t,t)$, define
\begin{align}
    Z_{k,t}^{d}(x_t)
    &:=\sum_{a\in\alphabet}
      p_{\mathrm{ref},1\mid t}^{d}(a\mid x_t)
      G_{k,t}^{\star,d}(a,x_t),
    \notag\\
    C_{k,t}^{d}(x_t)
    &:=\max_{z\in\alphabet}
      \frac{
        p_{\mathrm{ref},1\mid t}^{d}(z\mid x_t)
      }{
        p_{k-1,1\mid t}^{d}(z\mid x_t)Z_{k,t}^{d}(x_t)
      },
    \notag\\
    \mathcal C_{k,T}
    &:=\operatorname*{ess\,sup}_{
      \substack{0\leq t\leq T,\,x_t,\,d}}
      C_{k,t}^{d}(x_t).
    \notag
\end{align}
where the essential supremum is taken under the training-context law on
$[0,T]$. The posterior-support conditions above make these ratios well
defined. With
\begin{equation*}
    \mathcal A_{\tau,T}^2
    :=\int_0^T\frac{a_t^2}{\tau(t)}\dd t,
    \qquad
    \mathsf K_{k,T}:=D\mathcal A_{\tau,T}\sqrt{2\mathcal C_{k,T}},
\end{equation*}
assume $\mathcal A_{\tau,T},\mathcal C_{k,T}<\infty$. These are analysis
constants, not algorithmic hyperparameters. Unlike the corresponding integral
through $t=1$, $\mathcal A_{\tau,T}$ can be finite because the endpoint rate
singularity is excluded. The product
$\mathsf K_{k,T}\sqrt{\Delta\mathcal L_k}$ is invariant under the common guide
rescaling induced by any $\widehat c_k>0$ in \eqref{eq:bridge-target}.

Let $\mathbb E_{\mathrm{tr},k}$ average over $t\sim\tau$,
$X_t\sim p_{k-1,t}$, and a uniformly chosen coordinate $d$.

\begin{proof}[Proof of Proposition~\ref{prop:bregman-posterior-calibration}]
Fix a training context $y=(d,x_t,t)$ and use the quantities above. Set
$G^\star=G_{k,t}^{\star,d}(\cdot,x_t)$,
$G=G_{\psi_k,t}^{d}(\cdot,x_t)$, and $u(z)=G(z)/G^\star(z)$. By
\eqref{eq:fkl-excess-risk}, the conditional contribution to population excess risk is
\begin{equation*}
    \Delta_y
    =\sum_{z\in\alphabet}
      p_{k-1,1\mid t}^{d}(z\mid x_t)
      G^\star(z)\phi(u(z)),
\end{equation*}
Moreover,
\begin{equation*}
    \bar q(z)=
    \frac{p_{\mathrm{ref},1\mid t}^{d}(z\mid x_t)G^\star(z)}{Z^\star},
    \qquad
    q_{\psi_k}(z)
    =\frac{\bar q(z)u(z)}{\mathbb E_{\bar q}[u]},
\end{equation*}
where $Z^\star=Z_{k,t}^{d}(x_t)$. Equation~\plaineqref{eq:operational-stage-posterior}
identifies $\bar q$ as the posterior induced by $G_k^\star$ and the frozen
posterior. Therefore
\begin{align}
    \kld(\bar q\Vert q_{\psi_k})
    &=\log\mathbb E_{\bar q}[u]-\mathbb E_{\bar q}[\log u] \notag\\
    &\leq\mathbb E_{\bar q}[u-1-\log u] \notag\\
    &=\frac{1}{Z^\star}\sum_{z}
      p_{\mathrm{ref},1\mid t}^{d}(z\mid x_t)
      G^\star(z)\phi(u(z)) \notag\\
    &\leq
      \left[
        \max_z
        \frac{p_{\mathrm{ref},1\mid t}^{d}(z\mid x_t)}{
          p_{k-1,1\mid t}^{d}(z\mid x_t)Z^\star}
      \right]\Delta_y,
    \notag
\end{align}
where the first inequality is $\log v\leq v-1$. Pinsker's inequality consequently
gives
\begin{equation}
    \tv(\bar q,q_{\psi_k})
    \leq
    \left[
      \frac{\Delta_y}{2}
      \max_z
      \frac{p_{\mathrm{ref},1\mid t}^{d}(z\mid x_t)}{
        p_{k-1,1\mid t}^{d}(z\mid x_t)Z^\star}
    \right]^{1/2}.
    \label{eq:conditional-bregman-tv-calibration}
\end{equation}

Under the training-context law defined above, conditioning the Bregman excess risk first
on $(d,X_t,t)$ shows that
\begin{equation}
    \Delta\mathcal L_k
    =\mathbb E_{\mathrm{tr},k}[\Delta_y].
    \label{eq:excess-as-conditional-average}
\end{equation}
Changing from the training-context measure to the rate-weighted exact occupation
measure, applying \eqref{eq:conditional-bregman-tv-calibration}, and then using
Cauchy--Schwarz yields, with
$\omega_{k,t}=q_{k,t}/p_{k-1,t}$,
\begin{align}
    &\int_0^T a_t\,
      \mathbb E_{X_t\sim q_{k,t}}
      \left[\sum_{d=1}^D
      \tv\!\left(\bar q_{k,1\mid t}^{d},q_{\psi_k,1\mid t}^{d}\right)
      \right]\dd t \notag\\
    &\quad\leq
      \left\{
        \mathbb E_{\mathrm{tr},k}\!\left[
          \indc{t\leq T}
          \left(
            \frac{D a_t q_{k,t}(X_t)}
            {\tau(t)p_{k-1,t}(X_t)}
          \right)^2 C_{k,t}^{d}(X_t)
        \right]
      \right\}^{1/2}
      \sqrt{\frac{\Delta\mathcal L_k}{2}} \notag\\
    &\quad\leq
      D\left\{\mathcal C_{k,T}
      \int_0^T\frac{a_t^2}{\tau(t)}
      \mathbb E_{p_{k-1,t}}[\omega_{k,t}^2]\dd t\right\}^{1/2}
      \sqrt{\frac{\Delta\mathcal L_k}{2}} \notag\\
    &\quad\leq
      \frac{D\mathcal A_{\tau,T}\sqrt{\mathcal C_{k,T}}}
      {\sqrt{\ress_k(\Delta\alpha_k)}}
      \sqrt{\frac{\Delta\mathcal L_k}{2}}.
    \label{eq:rate-weighted-fit-bound}
\end{align}
The second inequality expands the training expectation and bounds
$C_{k,t}^{d}\leq\mathcal C_{k,T}$; the last uses
\eqref{eq:noised-occupation-second-moment}. All posterior arguments are
$(\cdot\mid X_t)$. The indicator restricts the first Cauchy--Schwarz factor
to $[0,T]$, while its second factor is bounded by the full nonnegative excess
risk in \eqref{eq:excess-as-conditional-average}. Thus
\eqref{eq:rate-weighted-fit-bound} controls frozen-posterior-to-learned posterior
error without requiring the frozen posterior to equal the posterior induced by
the current source.

\paragraph{From posterior error to finite-time sampling error.}
Let $\mu_{k,t}^{\mathrm{ct}}$ be the law generated by $\widehat Q_{k,t}$; the
law generated by $Q_{k,t}$ is $q_{k,t}$. Both start from
$q_{k,0}=\mu_{k,0}^{\mathrm{ct}}=\refdist$. For a zero-mass signed row $\xi$,
write $\|\xi\|_{\mathrm{TV}}:=\frac12\sum_y|\xi(y)|$. For every $T<1$, Duhamel's formula
and contraction of total variation under a Markov operator give
\begin{equation}
    \tv(q_{k,T},\mu_{k,T}^{\mathrm{ct}})
    \leq\int_0^T
      \mathbb E_{X_t\sim q_{k,t}}
      \left[
        \left\|
          (Q_{k,t}-\widehat Q_{k,t})(X_t,\cdot)
        \right\|_{\mathrm{TV}}
      \right]\dd t.
    \label{eq:ctmc-duhamel-tv}
\end{equation}
For any current state $x$, the diagonal generator entry is minus the total exit
rate. Hence the TV norm of the signed row difference is at most the sum of the
absolute off-diagonal rate differences. Using the exact and learned
posterior-marginal rates,
\begin{align}
    \left\|(Q_{k,t}-\widehat Q_{k,t})(x,\cdot)\right\|_{\mathrm{TV}}
    &\leq\sum_{d=1}^D\sum_{z\neq x^d}
      a_t\left|q_{k,1\mid t}^{d}(z\mid x)
      -q_{\psi_k,1\mid t}^{d}(z\mid x)\right| \notag\\
    &\leq2a_t\sum_{d=1}^D
      \tv\!\left(q_{k,1\mid t}^{d},q_{\psi_k,1\mid t}^{d}\right).
    \label{eq:generator-posterior-tv}
\end{align}
For each context, the posterior triangle inequality and
Proposition~\ref{prop:source-posterior-stability} give
\begin{align}
    \tv(q_{k,1\mid t}^{d},q_{\psi_k,1\mid t}^{d})
    &\leq
      \mathsf S_{k,t}^{d}(x_t)
      \tv(p_{k-1,1\mid t}^{d},\widehat p_{k-1,1\mid t}^{d})
      +\tv(\bar q_{k,1\mid t}^{d},q_{\psi_k,1\mid t}^{d}).
    \notag
\end{align}
Define the corresponding rate-weighted source-posterior residual by
\begin{equation}
    \delta_{k,T}^{\mathrm{src}}
    :=
    \int_0^T a_t\,
      \mathbb E_{X_t\sim q_{k,t}}
      \left[
        \sum_{d=1}^{D}\mathsf S_{k,t}^{d}(X_t)
        \tv\!\left(
          \widehat p_{k-1,1\mid t}^{d}(\cdot\mid X_t),
          p_{k-1,1\mid t}^{d}(\cdot\mid X_t)
        \right)
      \right]\dd t.
    \label{eq:source-residual-definition}
\end{equation}
Combining
\eqref{eq:ctmc-duhamel-tv}, \eqref{eq:generator-posterior-tv}, and
\eqref{eq:rate-weighted-fit-bound} yields
\begin{align}
    \tv(\mu_{k,T}^{\mathrm{ct}},q_{k,T})
    &\leq
    \frac{D\mathcal A_{\tau,T}\sqrt{2\mathcal C_{k,T}}}
    {\sqrt{\ress_k(\Delta\alpha_k)}}
    \sqrt{\Delta\mathcal L_k}
    +2\delta_{k,T}^{\mathrm{src}} \notag\\
    &=\frac{\mathsf K_{k,T}}
    {\sqrt{\ress_k(\Delta\alpha_k)}}
    \sqrt{\Delta\mathcal L_k}
    +2\delta_{k,T}^{\mathrm{src}},
    \label{eq:finite-time-practical-bound}
\end{align}

\paragraph{Posterior, simulation, and truncation terms.}
\label{app:practical-residuals}
Let
$\mu_k^{\mathrm{ct}}:=\lim_{t\uparrow1}\mu_{k,t}^{\mathrm{ct}}$ denote the
terminal law of the canonical $\gamma=1$ learned continuous-time process, and define
\begin{equation}
\begin{aligned}
    r_{k,T}^{\mathrm{pr}}
    :={}&2\delta_{k,T}^{\mathrm{src}}
      +\tv(\mu_k,\mu_k^{\mathrm{ct}})
      +\tv(\mu_k^{\mathrm{ct}},\mu_{k,T}^{\mathrm{ct}})\\
      &+\tv(q_{k,T},\widetilde\pi_k).
\end{aligned}
    \label{eq:practical-residual-definition}
\end{equation}
Thus $r_{k,T}^{\mathrm{pr}}$ contains the rate-weighted frozen-posterior
mismatch and three TV terms: numerical implementation error relative to the
$\gamma=1$ learned continuous-time process, truncation of that process at
$T$, and truncation of the exact stage process at $T$. It excludes the
Bregman fitting contribution. The last term is explicit under uniform
replacement:
\begin{equation*}
    \tv(q_{k,T},\widetilde\pi_k)
    \leq 1-\kappa_T^D
    \leq D(1-\kappa_T).
\end{equation*}
Indeed, the forward coupling selects the identity branch for every coordinate
with probability $\kappa_T^D$. The learned tail can also be bounded when the guide has finite
dynamic range. If
\begin{equation*}
    \rho_{k,T}:=
    \operatorname*{ess\,sup}_{\substack{T\leq t<1,\,x_t,\,d\\z,z'\in\alphabet}}
    \frac{G_{\psi_k,t}^{d}(z,x_t)}{G_{\psi_k,t}^{d}(z',x_t)}<\infty,
\end{equation*}
then the learned posterior assigns at most
$\rho_{k,T}(1-\kappa_t)/\kappa_t$ mass away from the current category. Its
total exit rate on $[T,1)$ is therefore at most
$D\rho_{k,T}\dot\kappa_t/\kappa_t$, which gives
\begin{equation*}
    \tv(\mu_k^{\mathrm{ct}},\mu_{k,T}^{\mathrm{ct}})
    \leq \min\!\left\{1,D\rho_{k,T}\log(1/\kappa_T)\right\}.
\end{equation*}
On the fixed interval $[0,T]$, standard finite-state CTMC consistency makes
$\tv(\mu_k,\mu_k^{\mathrm{ct}})$ vanish under simulation refinement when the
learned rates satisfy boundedness and time regularity
\citep{campbell2022continuous}. The triangle inequality gives
\begin{align}
    \tv(\mu_k,\widetilde\pi_k)
    &\leq
      \tv(\mu_k,\mu_k^{\mathrm{ct}})
      +\tv(\mu_k^{\mathrm{ct}},\mu_{k,T}^{\mathrm{ct}})
      +\tv(\mu_{k,T}^{\mathrm{ct}},q_{k,T})
      +\tv(q_{k,T},\widetilde\pi_k) \notag\\
    &\leq
    \frac{\mathsf K_{k,T}}
    {\sqrt{\ress_k(\Delta\alpha_k)}}
    \sqrt{\Delta\mathcal L_k}
    +r_{k,T}^{\mathrm{pr}}.
    \notag
\end{align}
The second line uses \eqref{eq:finite-time-practical-bound} and proves
\eqref{eq:one-stage-practical-bound}.
\end{proof}

\subsection{Multistage propagation and exact recovery}
\label{app:multistage-propagation}

\begin{proof}[Proof of the multistage claim in
Theorem~\ref{thm:posterior-terminal-stability}]
Let
$e_k:=\tv(\mu_k,\pi_k)$.
Proposition~\ref{prop:bregman-posterior-calibration} and the definition of
$r_{k,T}^{\mathrm{pr}}$ give
\begin{equation}
    \tv(\mu_k,\widetilde\pi_k)
    \leq \zeta_{k,T}
    :=\frac{\mathsf K_{k,T}}{\sqrt{\eta_k}}
      \sqrt{\Delta\mathcal L_k}+r_{k,T}^{\mathrm{pr}}.
    \label{eq:proof-local-stage-error}
\end{equation}
The triangle inequality and Lemma~\ref{lem:positive-tilt-stability} then give
\begin{align}
    e_k
    &\leq
      \tv\!\left(\mu_k,\tiltop_{w_k}(\mu_{k-1})\right)
      +\tv\!\left(
        \tiltop_{w_k}(\mu_{k-1}),
        \tiltop_{w_k}(\pi_{k-1})
      \right) \notag\\
    &\leq
      \zeta_{k,T}
      +\Lambda_k e_{k-1}.
    \label{eq:proof-stagewise-recursion}
\end{align}
Equations~\eqref{eq:proof-local-stage-error} and
\eqref{eq:proof-stagewise-recursion} prove
\eqref{eq:recursive-error-bound}.
\end{proof}

\paragraph{Distribution-free envelope.}
Write
$D_\infty(P\Vert Q):=\log\max_{x:Q(x)>0}P(x)/Q(x)$ and
$\osc(E):=\max_xE(x)-\min_xE(x)$.
Iterating the recursion gives
$e_K\leq\sum_{j=1}^K\zeta_{j,T}\prod_{\ell=j+1}^K\Lambda_\ell$.
Since
$\Lambda_\ell\leq\max_x w_\ell(x)/\pi_{\ell-1}(w_\ell)
=\exp\{D_\infty(\pi_\ell\Vert\pi_{\ell-1})\}$, aligned energy minimizers and
telescoping normalizers along the shared-energy path imply the optional bound
\begin{equation*}
    e_K\leq\sum_{j=1}^K \zeta_{j,T}
    \exp\!\left\{D_\infty(\targetdist\Vert\pi_j)\right\},
    \qquad
    D_\infty(\targetdist\Vert\pi_j)
    \leq(1-\alpha_j)\osc(E).
\end{equation*}
This is a uniform envelope over all discrepancies with the same stagewise TV,
rather than a prediction of the realized terminal error. Attaining it would
require the signed discrepancy at each stage to concentrate on states with the
largest remaining normalized weight and to remain aligned with those
extremizers through every subsequent tilt. When the learned error is dispersed
across states, this adversarial alignment is absent, and the realized
amplification can be substantially smaller. We therefore use
\eqref{eq:recursive-error-bound}, with its stage-local normalizers
$\Lambda_k$, as the primary stability statement and retain the
$D_\infty$ expression only as a distribution-free fallback.

\begin{proof}[Proof of Corollary~\ref{cor:ideal-recovery}]
Proceed by induction. The base case is
$\mu_0=\pi_0=\refdist$, whose posterior is the analytic fixed base. Suppose
$\mu_{k-1}=\pi_{k-1}$ and the frozen guide induces the exact forward posterior
of this law on all rate-relevant times. By the assumed population optimality,
$G_{\psi_k}=G_k^\star$ on those times.
Proposition~\ref{thm:stage-exactness} then makes the learned and exact
posterior-marginal generators equal almost everywhere along the reverse path.
Since they share the initial law, their marginals agree for every $t<1$:
\begin{equation*}
    \mu_{k,t}^{\mathrm{ct}}=q_{k,t}.
\end{equation*}
Exact terminal simulation therefore gives
\begin{equation*}
    \mu_k
    =\mu_k^{\mathrm{ct}}
    =\lim_{t\uparrow1}q_{k,t}
    =\widetilde\pi_k
    =\tiltop_{w_k}(\mu_{k-1})
    =\tiltop_{w_k}(\pi_{k-1})
    =\pi_k.
\end{equation*}
Moreover, the population-optimal guide induces the exact posterior of $\pi_k$
under the shared forward kernel, so its frozen posterior satisfies the equality
required at stage $k+1$. The claim follows for all $k$, and
$\pi_K=\targetdist$ by
\eqref{eq:ideal-annealed-family}.
\end{proof}

\section{Practical Training Components}
\label{app:implementation}

We use two optimization aids on the large-lattice benchmarks: analytic local
preconditioning and conditional-reweight auxiliary supervision. Their roles
relative to exactness are different. Preconditioning is a positive
reparameterization and does not change the population optimum.
Conditional reweighting instead supplies an auxiliary signal;
for a learned source its implemented teacher is generally approximate and is not part of the exactness claim.

\subsection{Analytic local preconditioning}
\label{app:analytic-preconditioning}

For the uniform-replacement process, the analytic base posterior is given by
\eqref{eq:analytic-reference-posterior}. On the $16\times16$ lattice models, we
use local message passing to absorb the dominant short-range interactions before
learning. This is useful because the exact accumulated posterior can become
sharply concentrated even when the remaining correction is comparatively
smooth. Concretely, we factor the accumulated guide as
\begin{equation*}
    \log G_{\psi_k,t}^{i}(a,x_t)
    =\log G_{\mathrm{pre},k,t}^{i}(a,x_t)
     +\Delta_{\psi_k,t}^{i}(a,x_t),
\end{equation*}
and train the network to predict the residual $\Delta_{\psi_k}$. The fixed positive
field $G_{\mathrm{pre},k}$ is a guide, not itself a posterior: after
normalization, $p_{\mathrm{ref},1\mid t}^{i}G_{\mathrm{pre},k,t}^{i}$ locally
approximates the stage-$k$ posterior. Thus preconditioning changes the
parameterization, not the bridge objective or its population target.

For the lattice models, we index the same Boltzmann family directly by inverse
temperature, $\pi_\beta(x)\propto\exp\{-\beta H(x)\}$, instead of the generic
$\alpha$ notation; the Ising and Potts Hamiltonians are given in
Sections~\ref{sec:exp-ising4} and~\ref{sec:exp-lattices}. At a segment ending at inverse
temperature $\beta_k$, the preconditioner uses the tempered interaction
strength $\bar\beta_k:=\rho_{\mathrm{pre}}\beta_k$. Its auxiliary pairwise
posterior is
\begin{equation*}
    \widetilde p_{\bar\beta_k,1\mid t}(x_1\mid x_t)
    \propto
    \prod_i
      p_{\mathrm{ref},1\mid t}^{i}(x_1^i\mid x_t)e^{u_k(x_1^i)}
    \prod_{(i,j)\in\mathcal E}\Psi_k(x_1^i,x_1^j).
\end{equation*}
Here $\mathcal E$ is the periodic four-neighbor edge set. For Ising,
$s(a)=2a-1$, $u_k(a)=\bar\beta_khs(a)$, and
$\Psi_k(a,b)=e^{\bar\beta_kJs(a)s(b)}$; for the zero-field Potts model,
$u_k(a)=0$ and $\Psi_k(a,b)=e^{\bar\beta_kJ\indc{a=b}}$. Here $e^{u_k}$ and $\Psi_k$ 
are respectively the unary and pairwise factors of $\exp\{-\rho_{\mathrm{pre}}\beta_k H\}$. 
Thus $\rho_{\mathrm{pre}}$ multiplies every physical log-potential used by the fixed
BP approximation: $\rho_{\mathrm{pre}}=1$ uses the full endpoint Hamiltonian,
whereas $0<\rho_{\mathrm{pre}}<1$ supplies a softened local approximation.
This scale acts only inside the analytic preconditioner. It changes neither the
physical target $\pi_{\beta_k}$ nor the bridge increment or sampling guidance
strength; the learned residual is always trained against the full target.

We approximate the single-site marginals of this law using $M$ synchronous
cavity updates:
\begin{align*}
    b_{j\setminus i}^{(\ell)}(b)
    &:=\frac{
      p_{\mathrm{ref},1\mid t}^{j}(b\mid x_t)e^{u_k(b)}
      \prod_{r\in\mathcal N(j)\setminus\{i\}}m_{r\to j}^{(\ell-1)}(b)
    }{
      \sum_{b'\in\alphabet}
      p_{\mathrm{ref},1\mid t}^{j}(b'\mid x_t)e^{u_k(b')}
      \prod_{r\in\mathcal N(j)\setminus\{i\}}m_{r\to j}^{(\ell-1)}(b')}, \\
    m_{j\to i}^{(\ell)}(a)
    &:=\sum_{b\in\alphabet}
      b_{j\setminus i}^{(\ell)}(b)\Psi_k(a,b), \\
    \log G_{\mathrm{pre},k,t}^{i}(a,x_t)
    &:=u_k(a)+\sum_{j\in\mathcal N(i)}\log m_{j\to i}^{(M)}(a),
\end{align*}
where $\mathcal N(i)$ is the neighborhood of site $i$ and
$m_{j\to i}^{(0)}(a)=1$. The final line is the BP marginal divided by the
analytic base posterior, up to a category-independent scale. For Potts, the
message update simplifies to
\[
    m_{j\to i}^{(\ell)}(a)
    =1+\bigl(e^{\bar\beta_kJ}-1\bigr)b_{j\setminus i}^{(\ell)}(a).
\]
The canonical IEDG lattice runs use $M=2$ undamped updates. Because the
periodic lattice is loopy, the resulting local approximation can become overly
sharp as correlations strengthen. We therefore use
$\rho_{\mathrm{pre}}=1$, $0.5$, and $0.3$ in the disordered, near-critical,
and ordered regimes, respectively, leaving the learned residual to recover the
omitted interaction and loop-dependent structure. The near-critical scale is
examined in Appendix~\ref{app:component-ablations}.

\subsection{Conditional-reweight auxiliary supervision}
\label{app:conditional-reweight}

For a fixed frozen source-posterior teacher and stage normalizer
$\widehat c_k$, the sampled response in \eqref{eq:bridge-target} is an
unbiased one-sample target for the conditional mean induced by the frozen
posterior, but it can have high conditional variance. In the
near-critical regime, the practical loss is therefore
\begin{equation*}
    \mathcal L_k^{\mathrm{train}}[G_\psi]
    =
    \mathcal L_k[G_\psi]
    +\lambda_{\mathrm{CR}}\mathcal L_k^{\mathrm{CR}}[G_\psi].
\end{equation*}
Here $\mathcal L_k$ is the core bridge objective in
\eqref{eq:bridge-objective}, and $\lambda_{\mathrm{CR}}$ controls an auxiliary
category-wise loss:
\begin{equation*}
\begin{aligned}
    \mathcal L_k^{\mathrm{CR}}[G_\psi]
    :=
    \mathbb E_{(t,X_t)\sim\mathcal D_k^{\mathrm{CR}}}
    \bigg[
      \frac1D\sum_{d=1}^{D}\sum_{z\in\alphabet}
      &\widehat p_{\beta_s,1\mid t}^{\mathrm{CR},d}(z\mid X_t)\\[-2pt]
      {}\times&
      \ell_{\mathrm{DGM}}\!\left(
        G_{\psi,t}^{d}(z,X_t),
        \widehat G_{k,t}^{\mathrm{CR},d}(z,X_t)
      \right)
    \bigg].
\end{aligned}
\end{equation*}
The refreshed buffer $\mathcal D_k^{\mathrm{CR}}$ contains noised contexts from
the realized source buffer and the conditional estimates defined below. The
main experiments enable this term only in the near-critical Ising and Potts
$16\times16$ regimes, with $\lambda_{\mathrm{CR}}=1$.

For each context, the auxiliary accumulated-guide teacher is
\begin{equation*}
    \widehat G_{k,t}^{\mathrm{CR},d}(z,x_t)
    :=
    \frac{
      \widehat p_{\beta_s,1\mid t}^{\mathrm{CR},d}(z\mid x_t)
    }{
      p_{\mathrm{ref},1\mid t}^{d}(z\mid x_t)
    }
    \widehat h_{k,t}^{\mathrm{CR},d}(z,x_t).
\end{equation*}
The first factor estimates the source-posterior correction relative to the
analytic base; the second estimates the normalized incremental tilt. Given
$S$ post-burn-in Gibbs records pooled across chains and retained sweeps, the
implementation computes
\begin{align*}
    \widehat p_{\beta_s,1\mid t}^{\mathrm{CR},d}(z\mid x_t)
    &:=
    \frac1S\sum_{m=1}^{S}
    q_{\beta_s,t}^{d}
    \bigl(z\mid X_{1,-d}^{(m)},x_t\bigr),
    \\
    \widehat h_{k,t}^{\mathrm{CR},d}(z,x_t)
    &:=
    \frac{
      \sum_{m=1}^{S}
      q_{\beta_s,t}^{d}
      \bigl(z\mid X_{1,-d}^{(m)},x_t\bigr)
      e^{-\Delta\beta_k H(X_{1,-d}^{(m)},z)}/\widehat c_k
    }{
      \sum_{m=1}^{S}
      q_{\beta_s,t}^{d}
      \bigl(z\mid X_{1,-d}^{(m)},x_t\bigr)
    }.
\end{align*}
Here $(x_{1,-d},z)$ denotes the state obtained by setting coordinate $d$ to
$z$. The records target the posterior induced by the ideal physical source
$\pi_{\beta_s}(x)\propto e^{-\beta_sH(x)}$:
\begin{equation*}
    \varrho_{\beta_s,t}^{x_t}(x_1)
    \propto
    e^{-\beta_s H(x_1)}p_{t\mid1}(x_t\mid x_1),
\end{equation*}
where $\beta_s$ is the source inverse temperature and
$\Delta\beta_k$ is the current temperature increment. Because the forward path
factorizes across coordinates, write $p_{t\mid1}^{d}$ for its coordinate
factor. The exact single-site Gibbs conditional is
\begin{equation*}
    q_{\beta_s,t}^{d}(z\mid x_{1,-d},x_t)
    =
    \frac{
      p_{t\mid1}^{d}(x_t^d\mid z)
      e^{-\beta_s H(x_{1,-d},z)}
    }{
      \sum_{a\in\alphabet}
      p_{t\mid1}^{d}(x_t^d\mid a)
      e^{-\beta_s H(x_{1,-d},a)}
    }.
\end{equation*}

To see why these estimates have the desired meaning, let
$\mathbb E_{\varrho}$ denote expectation under
$\varrho_{\beta_s,t}^{x_t}$ and set
$\bar h_{k,t}^{d}:=h_{k,t}^{d}/\widehat c_k$. Let
$p_{\beta_s,1\mid t}^{d}$ denote the coordinate posterior induced by
$\pi_{\beta_s}$. The tower property gives
\begin{align*}
    p_{\beta_s,1\mid t}^{d}(z\mid x_t)
    &=\mathbb E_{\varrho}\!\left[
      q_{\beta_s,t}^{d}(z\mid X_{1,-d},x_t)
    \right], \\
    \bar h_{k,t}^{d}(z,x_t)
    &=\frac{
      \mathbb E_{\varrho}\!\left[
        q_{\beta_s,t}^{d}(z\mid X_{1,-d},x_t)
        e^{-\Delta\beta_k H(X_{1,-d},z)}/\widehat c_k
      \right]
    }{
      \mathbb E_{\varrho}\!\left[
        q_{\beta_s,t}^{d}(z\mid X_{1,-d},x_t)
      \right]
    }.
\end{align*}
Thus both empirical quantities Rao--Blackwellize each Gibbs record over its
single-site conditional. The second is a finite-sample ratio estimator and is
not generally unbiased. Moreover, it uses the idealized physical source
$\pi_{\beta_s}$ rather than the realized learned source $\mu_{k-1}$; it is
therefore auxiliary supervision, not part of the exact bridge claim.

\section{Experimental Protocol and Reproducibility}
\label{app:experiments}

This appendix provides the information needed to reproduce the three empirical
questions in Section~\ref{sec:experiments}. We first summarize benchmark
configurations and construct the reference distributions, then specify the reported metrics, training and
model-selection protocol, baseline provenance, and held-out reporting.
IEDG's algorithms and optional training components appear in
Appendices~\ref{app:algorithms} and~\ref{app:implementation}; controlled studies
and additional diagnostics are collected in
Appendix~\ref{app:additional-experiments}.

\subsection{Benchmarks and reference construction}
\label{app:canonical-eval-policy}

Table~\ref{tab:benchmark-configs} summarizes the targets, references, and
terminal sampling budgets; NFE denotes posterior-network evaluations per
sample. Exact enumeration supplies the small-Ising law,
independent SW pools provide large-lattice references, and certified solutions
normalize Max-Cut quality. The generated-sample and NFE columns describe our
local IEDG-family evaluations; released baselines use the native generation
protocols summarized in Table~\ref{tab:baseline-provenance}.

\begin{table}[H]
    \centering
    \caption{Benchmark and terminal-evaluation configurations. ``Reference'' denotes
    the distribution used to compute evaluation errors.}
    \label{tab:benchmark-configs}
    \small
    \setlength{\tabcolsep}{4.2pt}
    \begin{tabular}{lcccccc}
        \toprule
        Benchmark & Size & $|\alphabet|$ & Target parameter & Reference & Generated & NFE \\
        \midrule
        Ising-exact & $4\times4$ & 2 & $\beta=.28,.4407,.6$ & Enumeration & $2^{20}$ & 128/256/256 \\
        Ising       & $16\times16$ & 2 & $\beta=.28,.4407,.6$ & SW & 4096 & 256 \\
        Potts       & $16\times16$ & 3 & $\beta=.5,1.005,1.2$ & SW & 4096 & 256 \\
        BA Max-Cut  & $[20,32]$ & 2 & $\lambda_{\rm cut}=5$ & Certified optimum
                     & 512/graph & 128 \\
                    & $[40,64]$ & 2 & $\lambda_{\rm cut}=5$ & Certified optimum
                     & 512/graph & 128 \\
                    & $[100,128]$ & 2 & $\lambda_{\rm cut}=5$ & Certified optimum
                     & 512/graph & 128 \\
        \bottomrule
    \end{tabular}
\end{table}

\paragraph{Lattice conventions.}
\label{app:lattice-benchmarks}
The Ising and Potts targets are defined in
Sections~\ref{sec:exp-ising4} and~\ref{sec:exp-lattices}. Both use periodic
square lattices with each undirected nearest-neighbor edge counted once.
Ising $4\times4$ uses $(J,h)=(1,0.1)$, Ising $16\times16$ uses
$(J,h)=(1,0)$, and Potts $16\times16$ uses $q=3$ and $J=1$. The inverse
temperatures in Table~\ref{tab:benchmark-configs} cover disordered,
near-critical, and ordered regimes.

\paragraph{Lattice references.}
\label{app:reference-protocol}
For Ising $4\times4$, we enumerate all $2^{16}$ configurations, evaluate
\eqref{eq:ising-benchmark} in float64, and normalize by log-sum-exp. Independent
i.i.d. sample sets of size $2^{20}$ from this law define the finite-sample
floor. Exact-MC entries report the mean and sample standard deviation over five
such sets, generated with seeds 2701--2705.
We pre-specify the ordered-regime endpoint $\beta=0.6$ as the primary Ising
$4\times4$ condition and report the remaining regimes in
Appendix~\ref{app:supplemental-diagnostics}.
For the $16\times16$ targets, we follow the SW reference protocol used by MDNS
\citep{zhu2025mdns}. Each pool comprises 128 separately initialized chains. We
discard 8,192 cluster sweeps per Ising chain and 65,536 per Potts chain, then
retain states at intervals of 128 sweeps. In the disordered and ordered regimes,
validation and test pools retain 32 states per chain ($4,096$ total). At Ising
$\beta=0.4407$ and Potts $\beta=1.005$, both pools retain 512 states per chain
($65,536$ total). Validation (seed 1) and test (seed 0) are separate SW runs with
independent chain initializations. The Ising-exact NFE entries in
Table~\ref{tab:benchmark-configs} follow the listed regime order.

Table~\ref{tab:sw-reference-stability} quantifies near-critical Monte Carlo
variation by treating the two independent SW pools as the distributions being
compared. Increasing the pool from 4,096 to 65,536 states reduces every
seed-to-seed discrepancy, motivating the larger pools at these targets.

\begin{table}[H]
    \centering
    \caption{Near-critical SW reference stability (seed 0 versus seed 1;
    zero is ideal). Metrics use the same fixed definitions as model evaluation;
    dashes denote inapplicable phase statistics.}
    \label{tab:sw-reference-stability}
    \scriptsize
    \setlength{\tabcolsep}{3.2pt}
    \begin{tabular}{lrrrrrr}
        \toprule
        Target & States & Mag. & Corr. agg. & $x_\uparrow$ JS
          & Mode $\ell_1$ & CV JS \\
        \midrule
        Ising $\beta=.4407$ & 4,096  & .00416 & .190 & .0144 & -- & -- \\
                             & 65,536 & .00201 & .0573 & .00102 & -- & -- \\
        Potts $\beta=1.005$ & 4,096  & .0635 & .143 & -- & .0244 & .0555 \\
                             & 65,536 & .0280 & .0183 & -- & .00497 & .00379 \\
        \bottomrule
    \end{tabular}
\end{table}

\paragraph{BA Max-Cut.}
\label{app:maxcut-protocol}
For an undirected graph $G=(V,E)$ and assignment $x\in\{0,1\}^{|V|}$, let
\begin{equation*}
    C_G(x)=\sum_{\{i,j\}\in E}\mathbf 1\{x_i\ne x_j\}
\end{equation*}
be the cut size. We use the three node ranges in Table~\ref{tab:benchmark-configs},
following the BA setting reported in PDNS
\citep{guo2025pdnsv1}. Within each range,
node count is sampled uniformly and graphs are generated by the Barab\'asi--Albert
model with $m=4$, where each newly added vertex attaches to $m$ existing
vertices. A separate model is trained on 1024 graphs, selected on 32 validation
graphs, and evaluated on 32 independently generated test graphs.

Writing the optimization loss as $E_G(x):=-C_G(x)$, the finite-temperature
sampling target is
\begin{equation}
    \pi_\lambda(x\mid G)
    \propto e^{-\lambda E_G(x)}=e^{\lambda C_G(x)},
    \qquad
    \lambda\in[0,\lambda_{\rm cut}],\quad \lambda_{\rm cut}=5.
    \label{eq:maxcut-raw-energy}
\end{equation}
Thus increasing $\lambda$ concentrates probability on assignments with larger
cuts.
We compute $C_G^\star$ with a Gurobi mixed-integer linear program. Binary
$x_i$ encode the partition and binary $y_{ij}$ encode
$|x_i-x_j|$ using the four standard linear inequalities; the objective maximizes
$\sum_{\{i,j\}\in E}y_{ij}$. Solves use zero requested MIP gap, no time limit,
and otherwise default Gurobi settings. All training, validation, and test
instances terminate with status \texttt{OPTIMAL} and zero recorded MIP gap;
certified test optima are used only for evaluation.

\subsection{Evaluation metrics}

The metrics follow the three empirical questions in the main text:
distribution-level recovery, lattice statistics and phase coverage, and
best-of-budget versus average Max-Cut quality.

\begin{table}[H]
    \centering
    \caption{Metric map. Main-text columns are the primary evidence for each
    benchmark; Appendix~\ref{app:supplemental-diagnostics} reports the listed
    complementary distributional diagnostics.}
    \label{tab:metric-map}
    \small
    \setlength{\tabcolsep}{5pt}
    \begin{tabular}{lll}
        \toprule
        Benchmark & Main-text metrics & Supplemental evidence \\
        \midrule
        Ising $4\times4$ & TV, KL, $\chi^2$ & regime sweep; observables \\
        Ising $16\times16$ & Mag., Corr. agg., $x_\uparrow$ JS
            & $J_E$, Corr.-curve; marginal plots \\
        Potts $16\times16$ & Mag., Corr. agg., Mode $\ell_1$
            & $J_E$, Corr.-curve, CV JS; marginal/CV plots \\
        BA Max-Cut & $R_{\max}$, $R_{\rm avg}$
            & stage-count and path-schedule ablations \\
        \bottomrule
    \end{tabular}
\end{table}

\paragraph{Exact-state metrics.}
\label{app:exact-metrics}
For terminal samples $X_1,\ldots,X_N$, define
$\widehat p_N(x)=N^{-1}\sum_n\mathbf 1\{X_n=x\}$. Against the exact target $\pi$,
we report
\begin{align*}
    \mathrm{TV}(\widehat p_N,\pi)
    &=\frac12\sum_x|\widehat p_N(x)-\pi(x)|, \\
    \mathrm{KL}(\widehat p_N\Vert\pi)
    &=\sum_{x:\widehat p_N(x)>0}\widehat p_N(x)
      \log\frac{\widehat p_N(x)}{\pi(x)}, \\
    \chi^2(\widehat p_N\Vert\pi)
    &=\sum_x\frac{(\widehat p_N(x)-\pi(x))^2}{\pi(x)}.
\end{align*}
Energy and nearest-neighbor-correlation errors are reported as supplemental
observable diagnostics.

\paragraph{Lattice observables.}
\label{app:mdns-metrics}
We use the MDNS magnetization and aggregate-correlation
definitions~\citep{zhu2025mdns}. Let $P$ denote the evaluated sample law and
$P_{\rm ref}$ the corresponding reference law. The Ising magnetization error is
\begin{equation*}
    \mathrm{Mag}_{\mathrm I}(P,P_{\rm ref})
    =\frac{1}{2L}\left[
      \sum_r|m_r^{\mathrm{row}}(P)-m_r^{\mathrm{row}}(P_{\rm ref})|
      +\sum_c|m_c^{\mathrm{col}}(P)-m_c^{\mathrm{col}}(P_{\rm ref})|
    \right].
\end{equation*}
Here
\begin{equation*}
    m_r^{\mathrm{row}}(P)=\frac1L\sum_c\mathbb E_P[s_{r,c}],
    \qquad
    m_c^{\mathrm{col}}(P)=\frac1L\sum_r\mathbb E_P[s_{r,c}]
\end{equation*}
are the row- and column-averaged spin profiles.

The corresponding Potts error is
\begin{equation*}
    \mathrm{Mag}_{\mathrm P}(P,P_{\rm ref})
    =\frac{1}{2L}\left[
      \sum_r|M_r^{\mathrm{row}}(P)-M_r^{\mathrm{row}}(P_{\rm ref})|
      +\sum_c|M_c^{\mathrm{col}}(P)-M_c^{\mathrm{col}}(P_{\rm ref})|
    \right],
\end{equation*}
where, for site $i=(r,c)$,
\begin{equation*}
    M_i(P)=\frac{q\max_aP(x_i=a)-1}{q-1},
    \qquad
    M_r^{\mathrm{row}}(P)=\sum_cM_{r,c}(P),
    \qquad
    M_c^{\mathrm{col}}(P)=\sum_rM_{r,c}(P).
\end{equation*}

The main-text correlation score is the MDNS-style aggregate error
\begin{equation*}
    \mathrm{CorrAgg}(P,P_{\rm ref})
    =\frac1{L^2}\sum_{k,l}\left(
      |C^{\mathrm{row}}_P(k,l)-C^{\mathrm{row}}_{P_{\rm ref}}(k,l)|
      +|C^{\mathrm{col}}_P(k,l)-C^{\mathrm{col}}_{P_{\rm ref}}(k,l)|
    \right).
\end{equation*}
Here
\begin{equation*}
    C^{\mathrm{row}}_P(k,l)=\sum_c K((k,c),(l,c);P),\qquad
    C^{\mathrm{col}}_P(k,l)=\sum_r K((r,k),(r,l);P),
\end{equation*}
with the family-specific kernels
\begin{align*}
    K_{\mathrm I}(i,j;P)
      &=\mathbb E_P[s_i s_j]-\mathbb E_P[s_i]\mathbb E_P[s_j], \\
    K_{\mathrm P}(i,j;P)
      &=P(x_i=x_j)-\frac1q.
\end{align*}
This is the empirical form of MDNS Eqs.~(28) and~(32) and is reported as
Corr. agg.

The supplemental pointwise correlation-curve error is
\begin{equation*}
    \mathrm{CorrCurve}_{\mathsf f}(P,P_{\rm ref})
    =\frac1{|\mathcal R_L|}\sum_{r\in\mathcal R_L}
      |c_{\mathsf f}(r;P)-c_{\mathsf f}(r;P_{\rm ref})|,
    \qquad \mathsf f\in\{\mathrm I,\mathrm P\},
\end{equation*}
where $\mathcal R_L=\{-L/2,\ldots,L/2-1\}$, $e_1$ is the first lattice axis,
and
\begin{align*}
    c_{\mathrm I}(r;P)
      &=\frac1{L^2}\sum_i\mathbb E_P[s_i s_{i+r e_1}], \\
    c_{\mathrm P}(r;P)
      &=\frac1{L^2}\sum_i\left[
        P(x_i=x_{i+r e_1})-\frac1q\right].
\end{align*}

\paragraph{Mode coverage.}
\label{app:distributional-metrics}
For normalized histograms $P,Q$ on shared bins, with $M=(P+Q)/2$, define
\begin{equation*}
    \operatorname{JS}(P,Q)
    =\frac12\operatorname{KL}(P\Vert M)
     +\frac12\operatorname{KL}(Q\Vert M).
\end{equation*}
The Ising phase metric is the JS divergence between the generated and SW laws
of
\begin{equation*}
    x_\uparrow(x)=\frac1{L^2}\sum_i\mathbf 1\{s_i=+1\}.
\end{equation*}
The main-text Potts phase error is
\begin{equation*}
    \mathrm{Mode}_{\mathrm P}(P,P_{\rm ref})
    =\left\|
      \operatorname{sort}(u(P))-
      \operatorname{sort}(u(P_{\rm ref}))
    \right\|_1,
\end{equation*}
where
\begin{equation*}
    u_a(P)=\mathbb P_{x\sim P}\!\left\{a=\arg\max_{b}f_b(x)\right\}.
\end{equation*}
Here $f_a(x)=L^{-2}\sum_i\mathbf 1\{x_i=a\}$, ties follow a fixed rule, and
all samples contribute to the score. For the supplemental order-parameter
marginal we use
\begin{equation*}
    m_{\mathrm P}(x)=\frac{q\max_a f_a(x)-1}{q-1}.
\end{equation*}
Potts also reports the MetaDNS two-dimensional collective variable
\begin{equation}
    z_1=f_1-\tfrac12(f_2+f_3),
    \qquad
    z_2=\tfrac{\sqrt3}{2}(f_2-f_3).
    \label{eq:potts-metadns-cv}
\end{equation}
CV JS is the JS divergence between generated and SW histograms on the fixed
$50\times50$ grid over $[-0.6,1.1]\times[-1,1]$, matching the projection,
range, and bin count in the official MetaDNS implementation
\citep{du2026metadns}.

All histogram grids are fixed independently of generated samples and shared
across methods, regimes, and seeds. We use 256 bins on $[-2,2]$ for Ising
energy per site, 256 bins on $[0,1]$ for $x_\uparrow$, and 256 bins on $[-2,0]$
for Potts energy per site; the energy intervals are the analytic supports of
the periodic $J=1,h=0$ models. Counts are normalized, natural logarithms are
used without a pseudocount, and the final bin includes its right endpoint.
Appendix~\ref{app:supplemental-diagnostics} reports energy JS, Corr.-curve MAE,
and Potts CV JS; the other histograms are visual diagnostics.

\paragraph{Max-Cut ratios.}
For the terminal sample set $\mathcal S_G$ of graph $G$, we compute
\begin{equation*}
    R_{\max}(G)=\max_{x\in\mathcal S_G}\frac{C_G(x)}{C_G^\star},
    \qquad
    R_{\mathrm{avg}}(G)=\frac1{|\mathcal S_G|}
      \sum_{x\in\mathcal S_G}\frac{C_G(x)}{C_G^\star}.
\end{equation*}
The first is best-of-budget optimization quality and the second is average sample
quality. Ratios are computed per graph before averaging across the test set.

\subsection{Models, training, and model selection}
\label{app:baseline-policy}

We describe the common IEDG configuration first, followed by path freezing,
terminal checkpoint selection, and the separate source-promotion decision used
by two near-critical transports. Baseline provenance is reported at the end of
the subsection.

\paragraph{IEDG configuration.}
All IEDG-family runs use the cosine uniform-replacement schedule
$\kappa_t=\sin^2(\pi t/2)$ with endpoint clamp $\epsilon=10^{-4}$, the
positive Bregman objective in
Algorithm~\ref{alg:iedg-training}, and the fixed stage normalizer in
Appendix~\ref{app:fixed-stage-normalizer}. Terminal samples are generated by the
direct-$q$ posterior-marginal sampler in Algorithm~\ref{alg:direct-q-sampling};
sample counts and NFE are given in Table~\ref{tab:benchmark-configs}.

\begin{table}[t]
    \centering
    \caption{Training configurations for the formal IEDG results. Each lattice
    row is the path segment ending at the reported $\beta$; ``updates'' is its
    realized segment budget or global cap. Width/depth/heads (W/D/H) describes the principal
    guide; LR and EMA denote learning rate and exponential-moving-average decay,
    and source buffer is the number of initially generated source states (per
    graph for Max-Cut).}
    \label{tab:training-configurations}
    \scriptsize
    \setlength{\tabcolsep}{3.2pt}
    \resizebox{\textwidth}{!}{%
    \begin{tabular}{lcccccccc}
        \toprule
        Reported target & Backbone & W/D/H & Batch & Updates & Optimizer
        & LR & EMA & Source buffer \\
        \midrule
        Ising $4\times4$, $\beta=.28$ & tiny MLP & -- & 512
          & 90000 & Adam & $10^{-4}$ & .999 & 1000000 \\
        Ising $4\times4$, $\beta=.4407$ & tiny MLP & -- & 1024
          & 65000 & Adam & $10^{-4}$ & .995 & 1000000 \\
        Ising $4\times4$, $\beta=.6$ & tiny MLP & -- & 1024
          & 70000 & Adam & $10^{-4}$ & .995 & 1000000 \\
        \midrule
        Ising $16\times16$, $\beta=.28$ & RoPE DeiT2D & 64/4/4 & 1024
          & 35000 & Adam & $4\!\times\!10^{-5}$ & .999 & 131072 \\
        Ising $16\times16$, $\beta=.4407$ & RoPE DeiT2D & 64/4/4 & 512
          & 65000 & Adam & $4\!\times\!10^{-5}$ & .999 & 100000 \\
        Ising $16\times16$, $\beta=.6$ & RoPE DeiT2D & 64/4/4 & 512
          & 65000 & Adam & $4\!\times\!10^{-5}$ & .999 & 100000 \\
        \midrule
        Potts $16\times16$, $\beta=.5$ & RoPE DeiT2D & 96/3/4 & 512
          & 40000 & AdamW & $3\!\times\!10^{-5}$ & .999 & 100000 \\
        Potts $16\times16$, $\beta=1.005$ & RoPE DeiT2D & 96/3/4 & 512
          & 90000 & AdamW & $3\!\times\!10^{-5}$ & .999 & 100000 \\
        Potts $16\times16$, $\beta=1.2$ & RoPE DeiT2D & 96/3/4 & 512
          & 46000 & AdamW & $3\!\times\!10^{-5}$ & .999 & 100000 \\
        \midrule
        BA Max-Cut (all ranges) & edge-aware DNFS & 128/3/4 & 512
          & 125000 & AdamW & $10^{-4}$ & .999 & 256/graph \\
        \bottomrule
    \end{tabular}%
    }
\end{table}

The lattice backbone is a two-dimensional RoPE DeiT with periodic spatial
encoding. The Max-Cut model is trained separately for each node range. It embeds
normalized degree, signed neighbor agreement, local cut fraction, time, and a
problem identifier, then applies three edge-aware attention/message-passing
blocks and a global pooled update to produce nodewise binary logits. Each graph
minibatch contains 32 graphs and 16 source states per graph.

\paragraph{Frozen ESS paths.}
\label{app:frozen-ess-path}
For each target sequence, the groupwise thresholds
$(\eta_{\mathrm{med}},\eta_{10})$ are held fixed across all stages of a
development run. We vary these thresholds only during development to control
the realized path resolution and stage count. At each stage, the empirical
rule in Appendix~\ref{app:ess-calibration} evaluates the current generated
source buffer and selects the next endpoint. Once a canonical path is chosen,
its endpoints are frozen and replayed in all formal training and evaluation
runs. The frozen endpoints are listed in
Table~\ref{tab:frozen-ess-paths}. Path and stage-count ablations use separately specified frozen alternatives,
while matched-$K$ comparisons preserve the realized stage count of the
canonical path. Together with the stage budgets in
the resolved configurations, they define the paper paths.
One-shot DGM has no annealing path, and the uniform-path study in
Appendix~\ref{app:path-ablation} uses the same realized stage count as its
ESS-controlled counterpart.

\begin{table}[t]
    \centering
    \caption{Frozen ESS-controlled path segments used for the formal IEDG
    results. Lattice endpoints are inverse temperatures $\beta$; Max-Cut
    endpoints are effective inverse temperatures $\lambda$ in
    Equation~\ref{eq:maxcut-raw-energy}. The adaptive Ising-$16\times16$, $\beta=.28$
    interior endpoints are rounded to four decimals.}
    \label{tab:frozen-ess-paths}
    \scriptsize
    \setlength{\tabcolsep}{5pt}
    \begin{tabular}{lcc}
        \toprule
        Reported target & Frozen segment endpoints & Stages $K$ \\
        \midrule
        Ising $4\times4$, $\beta=.28$ & $0,.1,.19,.28$ & 3 \\
        Ising $4\times4$, $\beta=.4407$ & $.28,.36,.4407$ & 2 \\
        Ising $4\times4$, $\beta=.6$ & $.4407,.50232,.6$ & 2 \\
        \midrule
        Ising $16\times16$, $\beta=.28$
          & $0,.0769,.1502,.2186,.28$ & 4 \\
        Ising $16\times16$, $\beta=.4407$
          & $.28,.3195,.3555,.3865,.413,.4407$ & 5 \\
        Ising $16\times16$, $\beta=.6$
          & $.4407,.472,.51,.552,.6$ & 4 \\
        \midrule
        Potts $16\times16$, $\beta=.5$ & $0,.1763,.343,.5$ & 3 \\
        Potts $16\times16$, $\beta=1.005$
          & $.5,.625,.738,.839,.92,.97,1.005$ & 6 \\
        Potts $16\times16$, $\beta=1.2$
          & $1.005,1.035,1.08,1.135,1.2$ & 4 \\
        \midrule
        BA Max-Cut (all ranges)
          & $0,.35,.798,1.3391,2.05,2.85,3.75,5$ & 7 \\
        \bottomrule
    \end{tabular}%
\end{table}

\paragraph{Terminal checkpoint selection.}
EMA checkpoints saved every 2,000 updates are evaluated at $\gamma=1$ on the
validation reference. Ising $4\times4$ minimizes TV. Ising $16\times16$
minimizes Corr.-curve MAE subject to Mag. constraint. Potts $16\times16$ minimizes Corr. agg.,
with Mag. constraint at $\beta=0.5$ and no constraint at the other targets. If no
large-lattice checkpoint satisfies an active constraint, the same primary
metric is minimized without it. Selection uses validation only; the chosen
checkpoint is then frozen for test evaluation. When another stage follows, its
$\gamma=1$ state also initializes the next-stage optimizer.

\paragraph{Near-critical source promotion.}
\label{app:source-promotion}
Validation-based source promotion is used only for Ising $16\times16$,
$0.28\!\rightarrow\!0.4407$, and Potts $16\times16$,
$0.5\!\rightarrow\!1.005$. Every 2,000 updates, candidate EMA checkpoints are
evaluated with $\gamma_{\mathrm{src}}\in\{1,1.1\}$ on the validation
SW pool. The same primary metric and active Mag. constraint select a
$(\text{checkpoint},\gamma_{\mathrm{src}})$ pair, which defines the next source
buffer and its frozen teacher; optimizer initialization uses the terminal
checkpoint selected above. Source promotion may use $\gamma_{\mathrm{src}}=1.1$;
all reported terminal samples use $\gamma_{\mathrm{term}}=1$.

\paragraph{Baseline provenance.}
Table~\ref{tab:baseline-provenance} distinguishes local results, released
checkpoints, and numbers from prior work. Local rows share our references and
metric code; published values are marked by $\dagger$ in the main tables.

\begin{table}[H]
    \centering
    \caption{Baseline provenance and terminal generation protocol.}
    \label{tab:baseline-provenance}
    \scriptsize
    \setlength{\tabcolsep}{4pt}
    \renewcommand{\arraystretch}{0.92}
    \begin{tabular}{llll}
        \toprule
        Comparison & Parameters or source & Terminal generation & Evidence \\
        \midrule
        Uniform (Max-Cut) & none & i.i.d. base samples & local \\
        Analytic preconditioner & fixed BP field, $\Delta_\psi\equiv0$
            & direct-$q$, matched samples/NFE & local \\
        One-shot DGM & trained by us & direct-$q$ & local \\
        UDNS & DASBS-based reproduction & factorized uniform CTMC & local \\
        MDNS (Ising $4\times4$) & value from MDNS & reported native setup & published$^\dagger$ \\
        MDNS ($16\times16$ lattices) & released checkpoints & native sampler & local \\
        MetaDNS ($16\times16$ lattices) & released checkpoints & native + reweight/resample & local \\
        PDNS (Max-Cut) & PDNS v1, Table~4 & -- & published$^\dagger$ \\
        \bottomrule
    \end{tabular}
\end{table}

The analytic-preconditioner row samples from the fixed BP field in
Appendix~\ref{app:analytic-preconditioning}, with no learned residual
($\Delta_\psi\equiv0$), using the same direct-$q$ sampler, terminal sample count,
and NFE as IEDG. IEDG and one-shot DGM use the same terminal sample count and NFE, apply no local
search after sampling, and set $\gamma_{\mathrm{term}}=1$. One-shot DGM uses the
same architecture family and aggregate optimizer budget, concentrated on a
single target. The UDNS rows in Tables~\ref{tab:ising4-exact}
and~\ref{tab:ising4-exact-sweep} are our local DASBS-based
reproduction~\citep{guo2026dasbs}, evaluated with the same exact-state metrics.

For the large lattices, MDNS and MetaDNS use their released checkpoints and
native samplers~\citep{zhu2025mdns,du2026metadns}. MetaDNS generates 4096
autoregressive proposals and converts them to 4096 unweighted samples by
normalized physical importance weighting followed by systematic resampling.
The published MDNS value for Ising $4\times4$ at $\beta=0.6$ uses the warm-start
setting reported in that paper. Max-Cut uses the published PDNS
results~\citep{guo2025pdnsv1}. We follow the PDNS v1 Table~4 benchmark specification, matching the BA graph
family, node-count ranges, sampling budget, and reported metrics. PDNS values
are quoted from that table; IEDG is evaluated on independently generated
held-out graphs under the same specification. Exact enumeration and SW are evaluation references, not learned baselines.

\subsection{Data separation and reporting}
\label{app:statistics}

Development uses generated source buffers and energy values to freeze the paths.
Validation uses SW seed 1 and rollout seed
1701 to select terminal checkpoints at $\gamma=1$ and, for
Section~\ref{app:source-promotion}, near-critical source pairs. After these
choices are frozen, test evaluation uses SW seed 0 and rollout seed 2701.
Reference-pool sizes are specified in Section~\ref{app:reference-protocol};
each large-lattice model rollout contains 4,096 states. Neither SW pool enters
the gradient objective or ESS path selection. Max-Cut test graphs and certified
optima never enter training or model selection.

\section{Additional Experimental Results}
\label{app:additional-experiments}

The evidence below addresses four questions left open by the main tables: how
many stages are useful, whether ESS-controlled scheduling helps at fixed stage
count, how sensitive IEDG is to its two large-lattice training components, and
which residual errors remain under distributional and configuration-level
checks.
All studies use the benchmark definitions and metrics in
Appendix~\ref{app:experiments}. Within each ablation, variants share
the optimizer budget, architecture, source-buffer size, terminal NFE, sample
count, and validation seeds; each path generates its own downstream buffers.
Throughout the ablations, $\theta$ denotes inverse temperature $\beta$ for the
lattices and effective inverse temperature $\lambda$ for Max-Cut. Phase error
denotes $x_\uparrow$ JS for Ising and Mode $\ell_1$ for Potts.

\subsection{How many annealing stages are needed?}
\label{app:stage-count-ablation}

Table~\ref{tab:stage-count-ablation} asks whether shorter bridges improve the
terminal law enough to justify additional stages. Here $K$ is the number of
transitions, so the path contains $K+1$ endpoints. We compare three explicit
paths for Ising $0.28\!\rightarrow\!0.4407$ and Potts
$0.5\!\rightarrow\!1.005$, and two for Max-Cut on
[100,128]. The canonical row uses the frozen
ESS-controlled path from the main experiment; alternatives change $K$ under
their specified stagewise schedules. Schedule design at fixed $K$ is examined
in Section~\ref{app:path-ablation}.

\begin{table}[H]
    \centering
    \caption{Sensitivity to the number of annealing stages. Paths are shown in
    the result rows (endpoints rounded to three decimals); lower lattice errors
    and higher Max-Cut ratios are better. The phase metric is $x_\uparrow$ JS
    for Ising and Mode $\ell_1$ for Potts.}
    \label{tab:stage-count-ablation}
    \scriptsize
    \setlength{\tabcolsep}{2.5pt}
    \resizebox{\textwidth}{!}{%
    \begin{tabular}{llcccc}
        \toprule
        Setting & $\{\theta_k\}_{k=0}^{K}$ & $K$
        & Mag. & Corr. agg. & Phase error \\
        \midrule
        Ising $0.28\!\to\!0.4407$
          & $\{.280,.329,.373,.409,.441\}$ & 4
          & $7.26{\times}10^{-3}$ & $5.38{\times}10^{-1}$
          & $6.37{\times}10^{-2}$ \\
          & $\{.280,.320,.356,.387,.413,.441\}$ & 5 (canonical)
          & $1.06{\times}10^{-2}$
          & $\mathbf{1.50{\times}10^{-1}}$ & $\mathbf{2.52{\times}10^{-2}}$ \\
          & $\{.280,.313,.344,.373,.398,.420,.441\}$ & 6
          & $\mathbf{5.42{\times}10^{-3}}$
          & $2.10{\times}10^{-1}$ & $3.75{\times}10^{-2}$ \\
        \midrule
        Potts $0.5\!\to\!1.005$
          & $\{.500,.637,.759,.865,.946,1.005\}$ & 5
          & $1.30{\times}10^{-1}$ & $6.32{\times}10^{-1}$
          & $3.91{\times}10^{-2}$ \\
          & $\{.500,.625,.738,.839,.920,.970,1.005\}$ & 6 (canonical)
          & $\mathbf{1.08{\times}10^{-1}}$
          & $\mathbf{1.23{\times}10^{-1}}$ & $\mathbf{1.09{\times}10^{-2}}$ \\
          & $\{.500,.598,.692,.775,.851,.914,.963,1.005\}$ & 7
          & $3.93{\times}10^{-1}$ & $4.91{\times}10^{-1}$
          & $9.62{\times}10^{-2}$ \\
        \bottomrule
    \end{tabular}
    }

    \vspace{0.6em}
    \begin{tabular}{llccc}
        \toprule
        Setting & $\{\theta_k\}_{k=0}^{K}$ & $K$
        & $R_{\max}$ & $R_{\mathrm{avg}}$ \\
        \midrule
        BA [100,128]
          & $\{.000,.544,1.261,2.256,3.444,5.000\}$
          & 5 & 0.884 & 0.838 \\
        & $\{.000,.350,.798,1.339,2.050,2.850,3.750,5.000\}$
          & 7 (canonical)
          & \textbf{0.915} & \textbf{0.871} \\
        \bottomrule
    \end{tabular}
\end{table}

Stage count is not monotone in accuracy. The canonical paths give the strongest
joint correlation and phase recovery, although the finer Ising path improves
Mag.; adding another Potts bridge weakens all three errors. Max-Cut instead
benefits from the denser path, showing that useful resolution depends on the
target rather than stage count alone.

\subsection{Does ESS-controlled scheduling help at fixed stage count?}
\label{app:path-ablation}

At the canonical stage count, we compare the frozen ESS-controlled path with a
uniform grid over the same endpoints and update budget. Let $K^\star$ be the
number of transitions in the ESS-controlled path for Ising
$0.28\!\rightarrow\!0.4407$ or Max-Cut. The uniform comparator uses the frozen grid
\begin{equation*}
    \theta^{\mathrm{uni}}_k
    =\theta_0+\frac{k}{K^\star}(\theta_{K^\star}-\theta_0),
    \qquad k=0,\ldots,K^\star.
\end{equation*}

\begin{table}[H]
    \centering
    \caption{ESS-controlled and uniform schedules at matched stage count and
    total training/evaluation budgets. Paths are rounded to three decimals.}
    \label{tab:path-ablation}
    \scriptsize
    \setlength{\tabcolsep}{2.5pt}
    \resizebox{\textwidth}{!}{%
    \begin{tabular}{lllcccc}
        \toprule
        Setting & Schedule & $\{\theta_k\}_{k=0}^{K}$ & $K$
        & Mag. & Corr. agg. & $x_\uparrow$ JS \\
        \midrule
        Ising $0.28\!\to\!0.4407$ & Matched-$K$ uniform
          & $\{.280,.312,.344,.376,.409,.441\}$ & 5
          & $\mathbf{5.58{\times}10^{-3}}$ & $3.17{\times}10^{-1}$
          & $2.61{\times}10^{-2}$ \\
        & ESS-controlled (canonical)
          & $\{.280,.320,.356,.387,.413,.441\}$ & 5
          & $1.06{\times}10^{-2}$ & $\mathbf{1.50{\times}10^{-1}}$
          & $\mathbf{2.52{\times}10^{-2}}$ \\
        \bottomrule
    \end{tabular}
    }

    \vspace{0.6em}
    \resizebox{\textwidth}{!}{%
    \begin{tabular}{lllccc}
        \toprule
        Setting & Schedule & $\{\theta_k\}_{k=0}^{K}$ & $K$
        & $R_{\max}$ & $R_{\mathrm{avg}}$ \\
        \midrule
        BA [100,128] & Matched-$K$ uniform
          & $\{.000,.714,1.429,2.143,2.857,3.571,4.286,5.000\}$
          & 7 & 0.903 & 0.864 \\
        & ESS-controlled (canonical)
          & $\{.000,.350,.798,1.339,2.050,2.850,3.750,5.000\}$
          & 7 & \textbf{0.915} & \textbf{0.871} \\
        \bottomrule
    \end{tabular}
    }
\end{table}

At matched stage count, uniform spacing gives the lower Ising Mag. error,
whereas ESS control gives better correlation and phase recovery. ESS control
also improves both Max-Cut ratios, supporting overlap-aware placement.

\subsection{Which practical components matter?}
\label{app:component-ablations}

\paragraph{Conditional reweighting.}
Table~\ref{tab:conditional-reweight-ablation} compares
$\lambda_{\mathrm{CR}}\in\{0.5,1.0,2.0\}$ on the Ising
$0.28\!\rightarrow\!0.4407$ and Potts $0.5\!\rightarrow\!1.005$ transitions.
The only configured change is the auxiliary-loss weight. $\lambda_{\mathrm{CR}}=1$ is canonical. We report the same $\gamma=1$, 4,096-sample protocol as in the main lattice tables.

\begin{table}[H]
    \centering
    \caption{Sensitivity to conditional-reweight strength. All other components,
    including analytic preconditioning and stagewise response normalization, remain
    enabled at their canonical settings.}
    \label{tab:conditional-reweight-ablation}
    \scriptsize
    \setlength{\tabcolsep}{4.5pt}
    \begin{tabular}{llccc}
        \toprule
        Setting & $\lambda_{\mathrm{CR}}$ & Mag. & Corr. agg. & Phase error \\
        \midrule
        Ising $0.28\!\to\!0.4407$ & 0.5 & $2.37{\times}10^{-2}$ & $3.14{\times}10^{-1}$ & $4.01{\times}10^{-2}$ \\
        & 1.0 (canonical) & $1.06{\times}10^{-2}$ & $\mathbf{1.50{\times}10^{-1}}$ & $2.52{\times}10^{-2}$ \\
        & 2.0 & $\mathbf{5.73{\times}10^{-3}}$ & $9.59{\times}10^{-1}$ & $\mathbf{1.89{\times}10^{-2}}$ \\
        \midrule
        Potts $0.5\!\to\!1.005$ & 0.5 & $1.99{\times}10^{-1}$ & $2.81{\times}10^{-1}$ & $1.24{\times}10^{-2}$ \\
        & 1.0 (canonical) & $\mathbf{1.08{\times}10^{-1}}$ & $\mathbf{1.23{\times}10^{-1}}$ & $\mathbf{1.09{\times}10^{-2}}$ \\
        & 2.0 & $3.31{\times}10^{-1}$ & $1.23$ & $1.53{\times}10^{-2}$ \\
        \bottomrule
    \end{tabular}
\end{table}

The canonical weight gives the best correlation recovery in both systems and
the best Potts phase balance. A larger weight improves the Ising phase error
but weakens correlations, while the weaker setting is less balanced.
$\lambda_{\mathrm{CR}}=1$ is a joint choice rather than a uniformly optimal
value for every statistic.

\paragraph{Analytic-preconditioner scale.}
We retain analytic preconditioning and vary only its strength on the same two
near-critical transitions. Table~\ref{tab:ratio-scale-ablation} compares
$\rho_{\mathrm{pre}}=\texttt{ratio\_scale}\in\{0.5,1.0\}$, with conditional
reweighting fixed to
$\lambda_{\mathrm{CR}}=1$.

\begin{table}[H]
    \centering
    \caption{Sensitivity to the analytic-preconditioner ratio scale in the
    near-critical regime; preconditioning remains enabled in every row.}
    \label{tab:ratio-scale-ablation}
    \scriptsize
    \setlength{\tabcolsep}{4.5pt}
    \begin{tabular}{llccc}
        \toprule
        Setting & $\rho_{\mathrm{pre}}$ & Mag. & Corr. agg. & Phase error \\
        \midrule
        Ising $0.28\!\to\!0.4407$ & 0.5 (canonical) & $\mathbf{1.06{\times}10^{-2}}$ & $\mathbf{1.50{\times}10^{-1}}$ & $\mathbf{2.52{\times}10^{-2}}$ \\
        & 1.0 & $1.73{\times}10^{-2}$ & $1.43$ & $8.92{\times}10^{-2}$ \\
        \midrule
        Potts $0.5\!\to\!1.005$ & 0.5 (canonical) & $\mathbf{1.08{\times}10^{-1}}$ & $\mathbf{1.23{\times}10^{-1}}$ & $\mathbf{1.09{\times}10^{-2}}$ \\
        & 1.0 & $3.38{\times}10^{-1}$ & $2.84$ & $4.26{\times}10^{-2}$ \\
        \bottomrule
    \end{tabular}
\end{table}

The half-scale setting gives lower errors across both transitions. This agrees
with the role of $\rho_{\mathrm{pre}}$ in Appendix~\ref{app:analytic-preconditioning}:
softening the local analytic field leaves the residual more freedom to recover
near-critical correlations and phase balance.

\subsection{What distributional errors remain?}
\label{app:supplemental-diagnostics}

\paragraph{Exact Ising distribution sweep.}
Table~\ref{tab:ising4-exact-sweep} reports the two Ising $4\times4$ conditions
omitted from the primary table. Together with Table~\ref{tab:ising4-exact}, it
gives the complete pre-specified regime sweep.

\begin{table}[H]
    \centering
    \caption{Full-distribution recovery in the remaining Ising $4\times4$
    regimes (lower is better). Exact MC is the equal-budget i.i.d.
    reference; UDNS is our DASBS-based reproduction and $\dagger$
    denotes values reported by MDNS. Bold marks the best neural sampler.}
    \label{tab:ising4-exact-sweep}
    \setlength{\tabcolsep}{7pt}
    \scriptsize
    \begin{tabular}{llccc}
        \toprule
        $\beta$ & Method & $\mathrm{TV}$
        & $\mathrm{KL}(\widehat p\Vert\pi)$
        & $\chi^2(\widehat p\Vert\pi)$ \\
        \midrule
        0.28 & Exact MC & $0.0664{\pm}0.00018$ & $0.0323{\pm}0.00021$
          & $0.0626{\pm}0.00064$ \\
        \cmidrule(lr){2-5}
        & One-shot DGM & 0.193 & 0.149 & 0.341 \\
        & UDNS (reproduced) & 0.0909 & 0.0412 & 0.0931 \\
        & MDNS ($F_{\rm WDCE}$)$^\dagger$ & \textbf{0.0799} & 0.0382 & 0.0868 \\
        & IEDG & 0.0860 & \textbf{0.0380} & \textbf{0.0823} \\
        \midrule
        0.4407 & Exact MC & $0.0221{\pm}0.00024$ & $0.0193{\pm}0.00010$
          & $0.0627{\pm}0.00285$ \\
        \cmidrule(lr){2-5}
        & One-shot DGM & 0.520 & 0.861 & 4.17 \\
        & UDNS (reproduced) & 0.249 & 0.416 & 22.0 \\
        & MDNS ($F_{\rm WDCE}$)$^\dagger$ & 0.0789 & 0.0375 & \textbf{0.0839} \\
        & IEDG & \textbf{0.0630} & \textbf{0.0300} & 0.116 \\
        \bottomrule
    \end{tabular}
\end{table}

\paragraph{Exact Ising observables.}
Table~\ref{tab:ising4-observables} tests whether the full-state gains translate
to familiar physical observables. It complements
Tables~\ref{tab:ising4-exact} and~\ref{tab:ising4-exact-sweep} with energy and
nearest-neighbor-correlation errors across the disordered, near-critical, and
ordered regimes; $C_{\rm nn}$
denotes the mean product of spins over undirected nearest-neighbor edges.

\begin{table}[H]
    \centering
    \caption{Observable errors on Ising $4\times4$ (lower is better). Exact MC
    is the equal-budget i.i.d. reference; bold marks the best neural sampler.}
    \label{tab:ising4-observables}
    \scriptsize
    \setlength{\tabcolsep}{6pt}
    \begin{tabular}{llcc}
        \toprule
        $\beta$ & Method & $|\Delta E|$ & $|\Delta C_{\rm nn}|$ \\
        \midrule
        0.28 & Exact MC & $0.00523{\pm}0.00609$ & $0.000150{\pm}0.000188$ \\
        \cmidrule(lr){2-4}
             & One-shot DGM & 3.53 & 0.108 \\
             & UDNS (reproduced) & 0.740 & 0.0224 \\
             & IEDG & \textbf{0.661} & \textbf{0.0203} \\
        \midrule
        0.4407 & Exact MC & $0.00940{\pm}0.00562$ & $0.000284{\pm}0.000177$ \\
        \cmidrule(lr){2-4}
               & One-shot DGM & 10.7 & 0.319 \\
               & UDNS (reproduced) & 1.32 & 0.0585 \\
               & IEDG & \textbf{0.876} & \textbf{0.0266} \\
        \midrule
        0.6 & Exact MC & $0.00355{\pm}0.00257$ & $0.000110{\pm}0.000056$ \\
        \cmidrule(lr){2-4}
            & One-shot DGM & 13.7 & 0.411 \\
            & UDNS (reproduced) & 0.470 & \textbf{0.00159} \\
            & IEDG & \textbf{0.298} & 0.00832 \\
        \bottomrule
    \end{tabular}
\end{table}

IEDG improves both observables over one-shot DGM throughout the sweep and over
UDNS in the disordered and near-critical regimes. In the ordered regime, IEDG
better matches energy while UDNS better matches nearest-neighbor correlation;
both remain above the Exact-MC floor.

\paragraph{Large-lattice residuals.}
The main table already measures local statistics and global phase coverage.
Here we report two residual diagnostics shared by both lattice families: energy
JS ($J_E$) tests the thermodynamic marginal, while Corr.-curve MAE tests spatial
dependence across separations. For Potts, fixed-grid CV JS additionally measures
the two-dimensional phase geometry. Tables~\ref{tab:large-lattice-distributions}
and~\ref{tab:potts-cv-js} evaluate these quantities on the same 4,096 generated
terminal samples per method; lower is better. Definitions and fixed evaluation
protocols are given in Appendix~\ref{app:mdns-metrics}.

\begin{table}[H]
    \centering
    \caption{Two complementary residual diagnostics for the $16\times16$ lattices:
    energy-marginal JS ($J_E$) and translational correlation-curve MAE
    (Corr.-curve; lower is better). Bold marks the best result.}
    \label{tab:large-lattice-distributions}
    \scriptsize
    \setlength{\tabcolsep}{3.8pt}
    \renewcommand{\arraystretch}{0.92}

    \textbf{(a) Ising $16\times16$}\par\vspace{0.25em}
    \begin{tabular}{llccccc}
        \toprule
        $\beta$ & Metric & Analytic precond. & MDNS & MetaDNS & One-shot DGM & IEDG \\
        \midrule
        0.28   & $J_E$       & 0.620 & \textbf{0.00230} & 0.00377 & 0.0123 & 0.00433 \\
               & Corr.-curve & 0.0462 & 0.00235 & \textbf{0.00125} & 0.0121 & 0.00682 \\
        \midrule
        0.4407 & $J_E$       & 0.693 & \textbf{0.00189} & 0.00548 & 0.693 & 0.0223 \\
               & Corr.-curve & 0.516 & 0.00467 & 0.00360 & 0.507 & \textbf{0.00275} \\
        \midrule
        0.6    & $J_E$       & 0.693 & \textbf{0.00117} & 0.00314 & 0.693 & 0.00445 \\
               & Corr.-curve & 0.848 & 0.00100 & 0.000908 & 0.803 & \textbf{0.000623} \\
        \bottomrule
    \end{tabular}

    \vspace{0.7em}
    \textbf{(b) Three-state Potts $16\times16$}\par\vspace{0.25em}
    \begin{tabular}{llccccc}
        \toprule
        $\beta$ & Metric & Analytic precond. & MDNS & MetaDNS & One-shot DGM & IEDG \\
        \midrule
        0.5   & $J_E$       & 0.0318 & 0.00334 & 0.00903 & 0.00246 & \textbf{0.00172} \\
              & Corr.-curve & 0.00390 & \textbf{0.000491} & 0.000973 & 0.000640 & 0.000627 \\
        \midrule
        1.005 & $J_E$       & 0.689 & \textbf{0.00330} & 0.00980 & 0.693 & 0.0413 \\
              & Corr.-curve & 0.325 & 0.00610 & 0.00379 & 0.339 & \textbf{0.00344} \\
        \midrule
        1.2   & $J_E$       & 0.693 & \textbf{0.00375} & 0.0146 & 0.693 & 0.0111 \\
              & Corr.-curve & 0.529 & \textbf{0.000261} & 0.00233 & 0.554 & 0.000699 \\
        \bottomrule
    \end{tabular}
\end{table}

\begin{table}[H]
    \centering
    \caption{Fixed-grid Potts CV JS across the disordered, near-critical, and
    ordered regimes (lower is better). Bold marks the best result.}
    \label{tab:potts-cv-js}
    \scriptsize
    \setlength{\tabcolsep}{5pt}
    \begin{tabular}{lccccc}
        \toprule
        $\beta$ & Analytic precond. & MDNS & MetaDNS & One-shot DGM & IEDG \\
        \midrule
        0.5   & 0.0143 & 0.0104 & 0.0119 & 0.00897 & \textbf{0.00868} \\
        1.005 & 0.646 & \textbf{0.0347} & 0.0548 & 0.658 & 0.0818 \\
        1.2   & 0.693 & \textbf{0.00332} & 0.0111 & 0.693 & 0.0329 \\
        \bottomrule
    \end{tabular}
\end{table}

Across these diagnostics, iteration largely repairs the near-critical
correlation errors of one-shot DGM and preserves accurate spatial structure in
near-critical and ordered regimes. Its remaining discrepancies lie mainly in the
near-critical and ordered energy marginals and Potts CV-plane geometry,
indicating residual phase-weight and within-mode errors rather than complete
mode collapse.

The following figures localize these residuals in the same frozen terminal
samples: Figure~\ref{fig:supplemental-lattice-distributions} shows marginal
errors, Figure~\ref{fig:supplemental-lattice-structure} resolves their spatial
dependence, Figure~\ref{fig:supplemental-potts-cv} shows Potts phase geometry,
and Figure~\ref{fig:ordered-endpoint-samples} gives a configuration-level check.

\begin{figure}[H]
    \centering
    \includegraphics[width=0.98\textwidth]{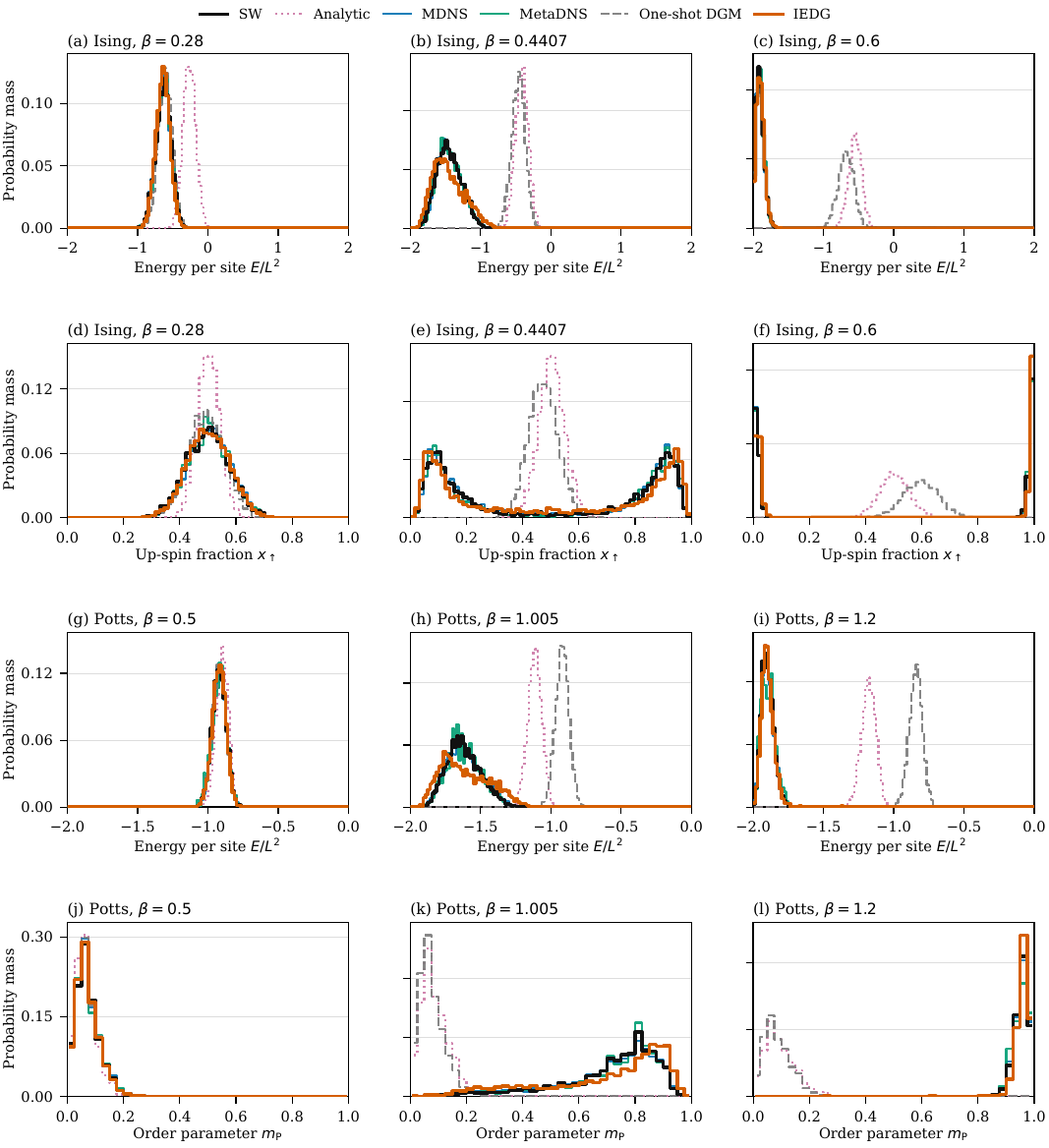}
    \caption{Regime-wide one-dimensional marginals from the frozen
    4,096-sample payloads. Columns traverse the disordered, near-critical, and
    ordered regimes; rows show Ising energy per site and $x_\uparrow$, followed
    by Potts energy per site and $m_{\mathrm P}$. All methods share histogram
    edges fixed from the analytic support, with SW (black) as the test
    reference.}
    \label{fig:supplemental-lattice-distributions}
\end{figure}

\begingroup
\setlength{\intextsep}{6pt}

\begin{figure}[H]
    \centering
    \includegraphics[width=0.90\textwidth]
        {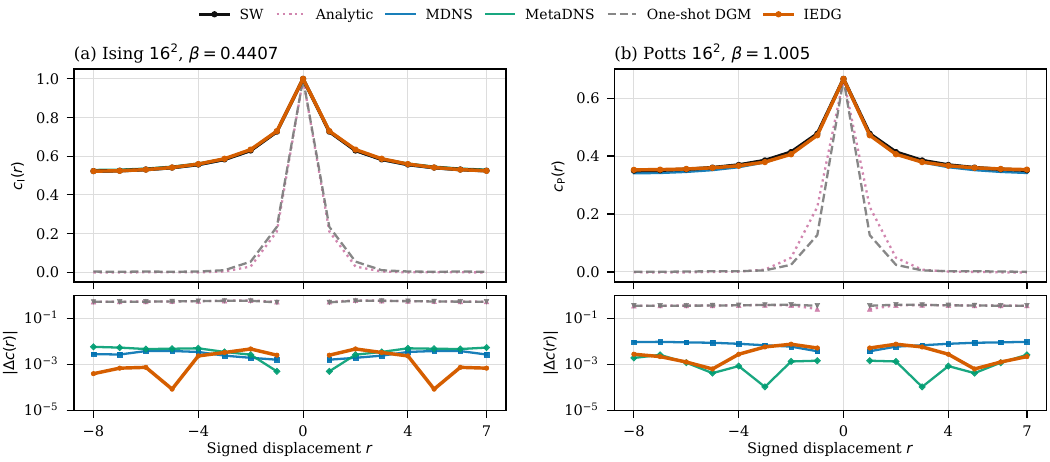}
    \vspace{-0.4em}
    \caption{Near-critical signed-displacement correlation profiles for Ising at $\beta=0.4407$ and Potts at $\beta=1.005$. Unlike Corr. agg., these profiles resolve spatial dependence at each separation. Upper panels compare each empirical curve with SW; lower panels show
    $|c_{\rm method}(r)-c_{\rm SW}(r)|$ on a logarithmic scale. The deterministic self-correlation at $r=0$ is omitted. Relative to one-shot DGM, the IEDG residual is reduced across separations rather than at a single displacement.}
    \label{fig:supplemental-lattice-structure}
\end{figure}
\vspace{-0.5em}

\begin{figure}[H]
    \centering
    \includegraphics[width=0.90\textwidth]
        {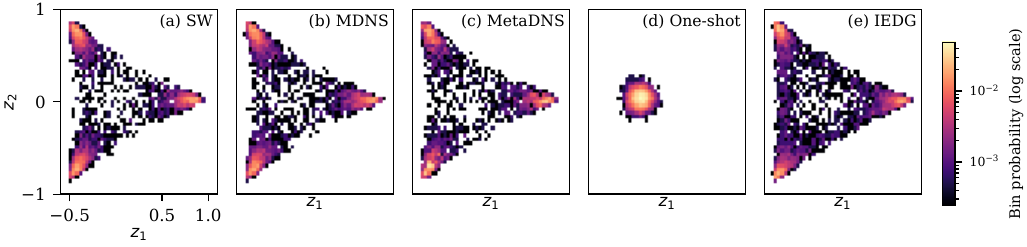}
    \vspace{-0.4em}
    \caption{Potts phase geometry at $\beta=1.005$. Unlike Mode $\ell_1$,
    which compares only ranked sector masses, the two-dimensional CV plane
    exposes within-sector geometry. IEDG recovers the three symmetry-related
    directions absent from one-shot DGM, while excess inter-sector probability
    explains its remaining CV JS gap. All methods use the same fixed
    $50\times50$ grid and logarithmic probability scale.}
    \label{fig:supplemental-potts-cv}
\end{figure}
\vspace{-0.5em}

\begin{figure}[H]
    \centering
    \includegraphics[width=0.90\textwidth]
        {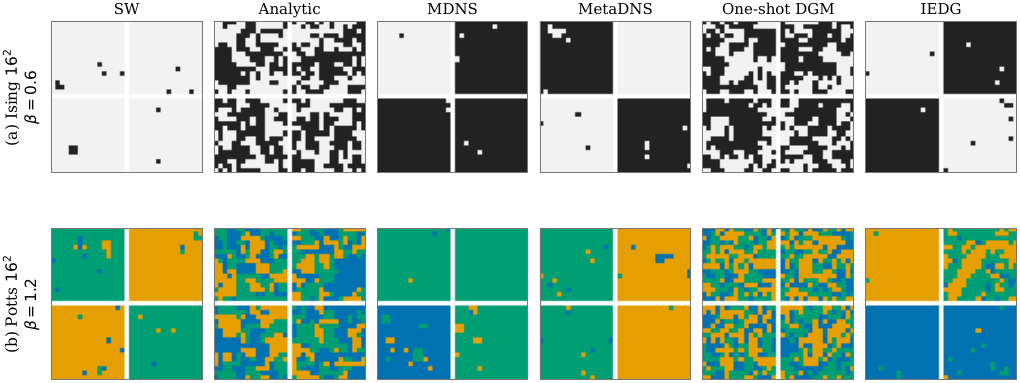}
    \vspace{-0.4em}
    \caption{Representative ordered-endpoint configurations. Columns denote
    methods; rows show Ising at $\beta=0.6$ and Potts at $\beta=1.2$.
    Indices $\{973,2172,2598,4081\}$ are drawn once with seed 2701 and reused
    across each method's independently generated 4,096-state payload. This qualitative check is
    not used for checkpoint selection; Potts colors are shared across methods.}
    \label{fig:ordered-endpoint-samples}
\end{figure}

\endgroup

\FloatBarrier

\end{document}